\documentclass{article}

\usepackage{arxiv}

\usepackage[utf8]{inputenc} 
\usepackage[T1]{fontenc}    
\usepackage{hyperref}       
\usepackage{url}            
\usepackage{booktabs}       
\usepackage{amsfonts}       
\usepackage{amssymb}
\usepackage{amsmath}
\usepackage{amsthm}
\usepackage{array}
\usepackage{balance}
\usepackage{nicefrac}       
\usepackage{microtype}      
\usepackage{cleveref}       
\usepackage{lipsum}         
\usepackage{graphicx}
\usepackage{subfigure}
\usepackage{algorithm}
\usepackage[noend]{algpseudocode}
\usepackage{setspace}
\usepackage{footnote}
\usepackage{tablefootnote}
\usepackage{footnote}
\usepackage{natbib}
\usepackage{doi}
\usepackage{enumitem}
\allowdisplaybreaks[4]

\newtheorem{Definition}{\bf{Definition}}

\newtheorem{Theorem}{\bf{Theorem}}
\newtheorem{Corollary}{\bf{Corollary}}
\newtheorem{Lemma}{\bf{Lemma}}
\newtheorem{Remark}{\bf{Remark}}
\newtheorem{Assumption}{\bf{Assumption}}

\newcommand{\Rcal}{{\mathcal{R}}}

\newcommand{\Ecal}{{\mathcal{E}}}

\newcommand{\Ncal}{{\mathcal{N}}}
\newcommand{\Vcal}{{\mathcal{V}}}
\renewcommand{\a}{{\bf a}}

\newcommand{\g}{{\bf g}}

\renewcommand{\u}{{\bf u}}
\renewcommand{\v}{{\bf v}}
\newcommand{\w}{{\bf w}}
\newcommand{\x}{{\bf x}}
\newcommand{\y}{{\bf y}}

\newcommand{\A}{{\bf A}}

\newcommand{\Dcal}{\mathcal{D}}

\newcommand{\I}{{\bf I}}
\newcommand{\J}{{\bf J}}

\newcommand{\N}{{\bf N}}

\renewcommand{\P}{{\bf P}}

\newcommand{\U}{{\bf U}}

\newcommand{\Ocal}[1]{{\mathcal{O}\left(#1\right)}}

\newcommand{\Xcal}{{\mathcal{X}}}

\newcommand{\bPi}{\boldsymbol{\Pi}}

\newcommand{\1}{{\bf 1}}

\newcommand{\argmin}{\operatornamewithlimits{\tiny argmin}}
\newcommand{\argmax}{\operatornamewithlimits{argmax}}

\newcommand{\lrincir}[1]{\left( #1 \right)}
\newcommand{\abs}[1]{\left\lvert#1\right\rvert}

\newcommand{\lrnorm}[1]{\left\lVert#1\right\rVert}
\newcommand{\lrangle}[1]{\left\langle#1 \right\rangle}

\newcommand{\EE}{\mathop{\mathbb{E}}}
\newcommand{\PP}{\mathop{\mathbb{P}}}
\newcommand{\RR}{\mathbb{R}}

\title{Markov Chain Mirror Descent On Data Federation}

\date{}

\author{Yawei Zhao \\ 
	Medical Big Data Research Center\\ 
	Chinese PLA General Hospital\\ 
	Beijing, China, 100851 \\ 
	\texttt{csyawei.zhao@gmail.com} \\ 
}

\renewcommand{\shorttitle}{Markov Chain Mirror Descent On Data Federation}

\hypersetup{
pdftitle={Model Wandering},
pdfsubject={cs},
pdfauthor={Yawei Zhao},
pdfkeywords={Markov chain, Mirror descent, Data Federation, Optimization, Federated learning},
}

\begin{document}
\maketitle

\begin{abstract}
Stochastic optimization methods such as mirror descent have wide applications due to low computational cost. Those methods have been well studied under assumption of the independent and identical distribution, and usually achieve sublinear rate of convergence. However, this assumption may be too strong and unpractical in real application scenarios. Recent researches investigate stochastic gradient descent when instances are sampled from a Markov chain. Unfortunately, few results are known for stochastic mirror descent. In the paper, we propose a new version of stochastic mirror descent termed by \textsc{MarchOn} in the scenario of the federated learning. Given a distributed network, the model iteratively travels from a node to one of its neighbours randomly. Furthermore, we propose a new framework to analyze \textsc{MarchOn}, which yields best rates of convergence for convex, strongly convex, and non-convex loss. Finally, we conduct empirical studies to evaluate the convergence of \textsc{MarchOn}, and validate theoretical results.    
\end{abstract}

\keywords{Markov chain, Mirror descent, Data federation, Convergence rate.}

\section{Introduction}
\label{sect:introduction}
Stochastic optimization methods have draw much attention in the era of big data due to their low computational cost \cite{yang2015big}. Examples of those methods include gradient descent, mirror descent, and their stochastic versions \cite{bubeck2015convex,hazan2016introduction,shalev2012online}. Generally, those stochastic optimization methods iteratively sample one instance randomly, and update a given model by using such instance. When those instances are assumed to be sampled from the independent and identical distribution (IID), those methods usually obtain sublinear rate of convergence \cite{books2014shalev}. 

However, the assumption of IID may be too strong and unpractical in real application scenarios \cite{tit2013agarwal,kuznetsov2020discrepancy}. A relatively mild assumption is that instances are sampled from a Markov chain. Under the assumption, extensive researches develop Markov chain versions of stochastic optimization methods, and furthermore investigate convergence rate of those methods \cite{duchi:2012:siamjoo,sun2018markov,even:2022:mcss:icml}. For example, $\Ocal{\frac{\ln T}{\sqrt{T}}}$ rate of convergence is achieved for convex loss \cite{duchi:2012:siamjoo}. $\Ocal{\frac{1}{T^{1-q}}}$ rate of convergence is obtained for either convex or non-convex loss \cite{sun2018markov}. $\Ocal{\frac{\ln T}{T}}$ rate of convergence is established for strongly convex loss, and $\Ocal{\frac{\ln T}{\sqrt{T}} + \frac{\ln T}{T}}$ rate of convergence is gained 
for non-convex loss \cite{even:2022:mcss:icml}. More results are summarized in Table \ref{table:convergence:rate}. Although extensive researches present more and more results on the convergence of the Markov chain stochastic optimization methods, those studies usually have three limitations. First, none of existing work offers a unified framework, providing convergence rate for convex, strongly convex and non-convex loss. Second, the existing result of convergence rate is sub-optimal for convex loss. Third, most existing researches focus on stochastic gradient descent, and few results are known for mirror descent. We are motivated by drawbacks of existing work, and provide a new analysis framework for Markov chain mirror descent. 

In the paper, we propose a new version of the Markov chain mirror descent method termed by \textsc{MarchOn} in the scenario of federated learning. In this scenario, the model travels over the distributed network randomly. When the model arrives at a node, it is updated by performing mirror descent with local data of the node. Since the model has to choose one of the node's neighbours as the destination node randomly, and travels one step to the destination node, the trajectory of node is a Markov chain. Furthermore, we provide a new unified framework to analyze the convergence rate of \textsc{MarchOn}, and present the best result on the convergence rate. Specifically, \textsc{MarchOn} achieves $\Ocal{\frac{1}{\sqrt{T}} + \frac{\ln T}{T} + \frac{1}{T}}$ rate of convergence for convex loss, which is the first result in the setting as far as we know. Finally, we conduct extensive empirical studies to evaluate the convergence of \textsc{MarchOn}, and validates theoretical results.

\section{Related Work}
\label{sec:related:work}

\subsection{Markov Chain Gradient Descent} 
Tao Sun et al. investigate the convergence of Markov chain stochastic gradient descent, and achieves $\Ocal{\frac{1}{T^{1-q}}}$ rate of convergence by choosing the step size as $\frac{1}{t^q}$ with $0.5<q<1$ for both convex and non-convex loss \cite{sun2018markov}. Based on this work, Puyu Wang et al. investigate stability and generalization of Markov chain stochastic gradient descent, and provide a comprehensive analysis for both minimization and minimax problems \cite{wang2022stability}. Additionally, Abhishek Roy et al. propose the moving-average based single-time scale algorithm to find the $\epsilon$-stationary point, and presents $\Ocal{\frac{1}{T^{0.4}}}$ rate of convergence for non-convex loss \cite{roy2022constrained}. Moreover, Ron Dorfman et al. propose an adaptive method for stochastic optimization with Markovian data, and do not require the knowledge of the mixing time \cite{dorfman2022adapting}. It achieves $\Ocal{\frac{\ln T}{\sqrt{T}}}$ rate of convergence for both convex and non-convex loss. Besides, Guy Bresler et al. focus on least square loss, and establish information theoretic minimax lower bounds when instances are sampled from a Markov chain \cite{bresler2020lsr}. Their analysis results demonstrate the optimization with Markovian data is generally more difficult than that with independent data, which is consistent with our analysis. Recently, Aleksandr Beznosikov et al. offer a unified analysis framework for accelerated version of first-order gradient methods with Markovian noise, and also present its extension to variational inequalities \cite{beznosikov2023first}.  Note that those existing researches mainly provide rigorous analysis of convergence rate for stochastic gradient descent. Unfortunately, those analysis cannot be extended to the mirror descent algorithms directly. We note that John C. Duchi et al. show $\Ocal{\frac{\ln T}{\sqrt{T}}}$ rate of convergence for Markov chain mirror descent \cite{duchi:2012:siamjoo}. As shown in Table \ref{table:convergence:rate}, the proposed method achieves better rate of convergence for convex loss. As far as we know, it is the first known result about convergence rate for either strongly convex or non-convex loss. 

\begin{table}
\centering
\caption{Convergence rates of the proposed method and other related methods. `Convex', `Strongly Convex', `Non-Convex' represents different conditions on the convexity of loss function. `Mirror Map' represents whether the model is updated by mirror descent.}
\renewcommand\arraystretch{2}
\begin{tabular}{c|c|c|c|c}
\hline
Algorithms                             & Convex                    & Strongly Convex    & Non-Convex            & {\small Mirror Map} \\ \hline \hline  
EMD, \cite{duchi:2012:siamjoo} & $\Ocal{\frac{\ln T}{\sqrt{T}}}$\tablefootnote{See Corollary 2 and Eq. 11 in \cite{duchi:2012:siamjoo}.} &  None                              & None & Yes \\ \hline
MCGD, \cite{sun:2018:mcgd} & $\Ocal{\frac{1}{T^{1-q}}}$\tablefootnote{Note that $\frac{1}{2}<q<1$ here.} &  None                              &  $\Ocal{\frac{1}{T^{1-q}}}$ & No \\ \hline
IASA, \cite{roy2022constrained} & None &  None                              &  $\Ocal{\frac{1}{T^{0.4}}}$ & No \\ \hline
MAG, \cite{dorfman2022adapting} & $\Ocal{\frac{\ln T}{\sqrt{T}}}$ &  None                              &  $\Ocal{\frac{\ln T}{\sqrt{T}}}$ & No \\ \hline
MC-SGD, \cite{even:2022:mcss:icml} & None                   &  $\Ocal{\frac{\ln T}{T}}$   & $\Ocal{\frac{\ln T}{\sqrt{T}}\mathrm{+}\frac{\ln T}{T}}$  & No \\ \hline
Markov SGD, \cite{doan:2023:mgd:tac} & $\Ocal{\frac{\ln(\ln T)\ln^2 T}{\sqrt{T}} \mathrm{+} \frac{\ln^2 T}{\sqrt{T}}}$ & $\Ocal{\frac{\ln T}{T} + \frac{1}{T^2}}$ &  None\tablefootnote{Although \cite{doan:2023:mgd:tac} provides $\Ocal{\frac{\ln T}{T} + \frac{1}{T^2}}$ rate for non-convex loss with the Polyak-{\L}ojasiewicz condition, the convergence rate is still unknown for general smooth non-convex loss in the work. } & No \\ \hline
\textbf{this paper}                               &   $\Ocal{\frac{1}{\sqrt{T}} + \frac{\ln T}{T} + \frac{1}{T}}$   & $\Ocal{\frac{\ln T}{T}+\frac{1}{T}}$  & $\Ocal{\frac{\ln T}{\sqrt{T}}+\frac{1}{\sqrt{T}}}$ & Yes \\ \hline
\end{tabular}
\label{table:convergence:rate}
\end{table}

\subsection{Optimization in Federated Setting}
Xiangru Lian et al. propose an asynchronous decentralized stochastic gradient decent algorithm, which is robust for communication in in heterogeneous environment, and achieves almost best convergence rate \cite{2017Asynchronous}. Kunal Srivastava et al. introduce decentralized optimization algorithm with random communication graph, where every node chooses its neighbours to update model \cite{srivastava2011distributed}. Both researches assume instances are sampled from IID. Different from them, more and more federated optimization methods are proposed and analyzed under the Markov chain sampling strategy. Hoi-To Wai establishes a finite time convergence analysis for stochastic gradient descent, and achieves $\Ocal{\frac{\ln T}{\sqrt{T}}}$ rate of convergence under assumptions of Markovian noise and time varying communication network \cite{wai2020consensus}. Ghadir Ayache et al. propose random walk stochastic gradient descent method in the decentralized network, and establishes $\Ocal{\frac{1}{T^{1-q}}}$ with $\frac{1}{2} < q < 1$ rate of convergence \cite{ayache2019random}. They furthermore focus on alleviating the communication bottleneck and data heterogeneity during training of model, and design near-optimal node sampling strategy for random walk of model \cite{ayache2023walk}.  Tao Sun et al. propose both zeroth order and first order versions of decentralized Markov chain gradient descent methods, and present convergence rate of those methods \cite{sun2023decentralizedsgdmc}. Additionally, Xianghui Mao et al. propose a random walk strategy for node selection, and perform ADMM on the node \cite{mao2020walkman}. When the loss is least square, the linear convergence is achieved. Similarly, the proposed method is also designed for the setting of federated learning. However, compared with those existing researches, the proposed method obtains more tight results on the rate of convergence. The superiority of the proposed analysis framework is briefly illustrated in Table \ref{table:convergence:rate}.

\section{Markov Chain Mirror Descent On Data Federation}
\label{sec:mcmd}

\subsection{Problem Settings}

In the paper, we investigate mirror descent in the scenario of federated learning. The concept of \textit{network} is used as an abstraction of data federation, where a \textit{node} represents an owner of some private data, and a model travels from a node to one of its neighbours over the network. Generally, we use $\Ncal=\{\Vcal, \Ecal\}$, $\Vcal$, and $\Ecal$ to represent the network, its set of nodes, and its set of edges, respectively. Suppose there exist $n$ nodes, that is $|\Vcal| = n$, and a node is identified by an id number which is chosen from $\{0,1,2,\cdots,n\}$. For example, a node $v_i\in\Vcal$ is identified by $i\in[n]$. As illustrated in Figure \ref{figure:algo:marchon}, the left shows the abstraction of a data federation consisting of four nodes, and the right shows a network $\Ncal$. In the paper, we focus on the following optimization problem. 
\begin{align}
\label{equa:full:objective}
\min_{x\in\RR^d} f(\x) := \frac{1}{n}\sum_{v\in\Vcal} f_v(\x), 
\end{align} where $f_v$ represents the local loss at the node $v$. Major notations used in the paper are summarized as follows.
\begin{itemize}[leftmargin=*]
\item $\Ncal$ represents a network. $\Vcal$ and $\Ecal$ represents its node set and edge set, respectively. $v\in\Vcal$ represents a node, and $n = |\Vcal|$ is the number of nodes in $\Vcal$. $v_{i_t}$ represents a node whose id is $i_t\in[n]$.
\item $\Dcal_v$ represents the local data of the node $v$. Similarly, $\Dcal_{i_t}$ represents the local data of the node $v_{i_t}$.
\item Denote the external variance between nodes by $\sigma_{\Vcal}^2 := \EE_{v} \lrnorm{\nabla f_{v}(\x) - \nabla f(\x)}^2$ where $v$ is chosen from a Markov chain, and the internal variance for the node $v\in\Vcal$ by $\sigma_v^2 := \EE_{\a} \lrnorm{\nabla f_{v}(\x;\a) - \nabla f_v(\x)}^2$ where $\a$ is uniformly sampled from $\Dcal_v$.
\item $\Phi$ is a distance function, and $B_{\Phi}(\u,\v)$ represents the Bregman divergence between $\u$ and $\v$.
\item $f_{i_t}$ is the local loss function for the node $v_{i_t}$, and $\nabla f_{i_t}$ represent its gradient.
\item $\EE$ represents to take mathematical expectation, and the default is to take mathematical expectation with respect to all random variables.
\end{itemize}

\subsection{Algorithm}
\label{subsect:algo}

Suppose the model $\x_t$ arrives at the $i_t$ node for the $t$-th iteration. Sample an instance $\a$ from the local data $\Dcal_{i_t}$, that is $\a\sim\Dcal_{i_t}$, we denote the decaying set $\Xcal_t$ by 
\begin{align}
\label{equa:decaying:set:W}
\Xcal_t := \left\{  \x : \lrnorm{\x - \lrincir{\x_t - \eta_t \cdot \nabla f_{i_t}(\x;\a\sim\Dcal_{i_t})}} \le \eta_t^2 \cdot \lrnorm{\nabla f_{i_t}(\x_t;\a\sim\Dcal_{i_t})}^2 \right\}.
\end{align} Furthermore, given a Markov chain $\zeta_T:=\{i_t\}_{t=1}^T$, we update the model $\x_t$ by performing 
\begin{align}
\label{equa:parameters:update} 
\x_{t+1}  = \argmin_{\x\in\Xcal_t} \lrangle{\nabla f_{i_t}(\x_t;\a\sim\Dcal_{i_t}), \x - \x_t} + \frac{1}{\eta_t} B_{\Phi}\lrincir{\x, \x_t},
\end{align} where $\eta_t$ is the step size, $i_t$ with $i_t\in[n]$ represents the node's id selected at the $t$-th iteration, and $\nabla f_{i_t}(\x_t;\a\sim\Dcal_{i_t})$ is a Markov chain stochastic gradient at the node $i_t$. $B_{\Phi}(\u,\v)$ is defined by
\begin{align}
\nonumber
B_{\Phi}(\u,\v) := \Phi(\u) - \Phi(\v) - \lrangle{\nabla \Phi(\v), \u-\v},
\end{align} where $\Phi$ is a given distance function. It represents the Bregman divergence between $\u$ and $\v$ \cite{parikh2014proximal}. Generally, $\Phi$ is $\mu_{\Phi}$-strongly convex, that is, for any $\x$ and $\y$, $B_{\Phi}(\x, \y) \ge \frac{\mu_{\Phi}}{2} \lrnorm{\x - \y}^2$.

\begin{figure}
\setlength{\abovecaptionskip}{0pt}
\setlength{\belowcaptionskip}{0pt}
\centering 
\includegraphics[width=0.98\columnwidth]{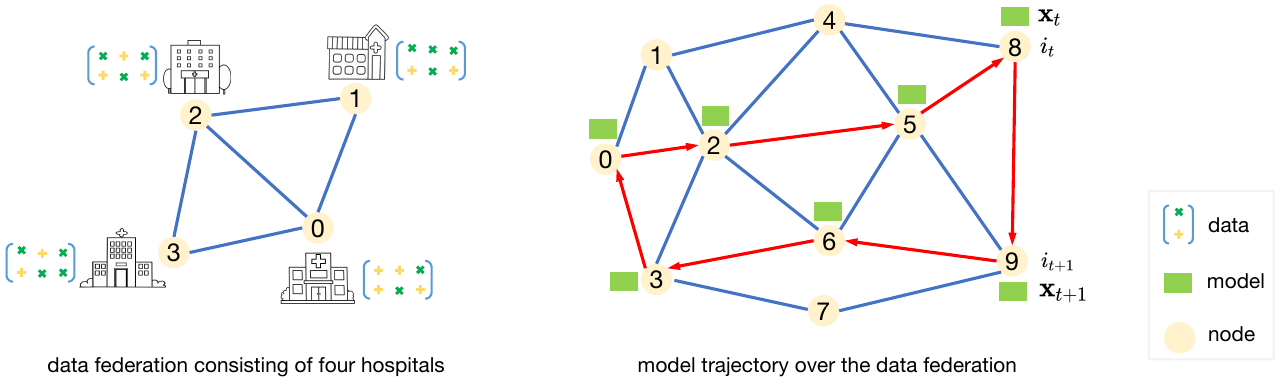}
\caption{Illustration of network, and the algorithm.}
\label{figure:algo:marchon} 
\end{figure}

\begin{algorithm}[h]
\setstretch{1.35}
    \caption{\textsc{MarchOn}: \textbf{Mar}kov \textbf{Ch}ain Mirror Descent \textbf{On} Data Federation.}
    \label{algo:marchon}
    \begin{algorithmic}[1]
        \Require The number of steps $T$, the network $\Ncal=\{\Vcal, \Ecal\}$, the initial model $\x_0$, and the initial starting node $v_{i_0}$.
        \State Initialize $v$ by letting $v \leftarrow v_{i_0}$.
        \For {$t= 0,1,2, ..., T-1$}
            \State The model chooses one of neighbours of $v$, e.g. $v_{i_t}\in\Vcal$, as the destination randomly, and then travels one step from $v$ to $v_{i_t}$.
                \State Sample an instance $\a\sim\Dcal_{i_t}$ of the node $v_{i_t}$ randomly, and compute a stochastic gradient $\nabla f_{i_t}(\x_t; \a)$.
                \State Perform \ref{equa:parameters:update} for once update of the model, and obtain $\x_{t+1}$.
                \State Update $v$ by $v\leftarrow v_{i_t}$.
        \EndFor
        \State \Return $\bar{\x}_T := \frac{1}{T}\sum_{t=1}^T \x_t$.
      \end{algorithmic}
\end{algorithm} 

As illustrated in Algorithm \ref{algo:marchon}, the model travels over the network (Line 3), and updates its values by using the local data of a node for every step (Line 4). Note that there are two levels of randomness, which is different from existing related researches \cite{sun:2018:mcgd,roy2022constrained,doan:2023:mgd:tac}. First, node is chosen from its neighbours randomly for the next step, and thus leads to a Markov chain in the trajectory. Second, an instance is chosen from the local data of a node randomly, and thus leads to a stochastic gradient. Comparing with existing federated optimization methods, the method has following advantages.
\begin{itemize}[leftmargin=*]
\item \textbf{Supporting mirror map.} Existing methods mainly focus on the Markov chain stochastic gradient descent \cite{sun:2018:mcgd,roy2022constrained,doan:2023:mgd:tac}. It is still unclear whether the mirror descent could work well in federated learning, and find the optimum by querying Markov chain stochastic gradient efficiently. The proposed method gives positive answers, and the regret analysis in the next section shows that it successfully finds the optimum with sublinear convergence.
\item \textbf{Low control complexity.} Optimization methods such as FedAvg \cite{yang2019federated}, and FedProx \cite{li2020federated} usually need to control either multiple or all nodes, use local data of those nodes, and then update model by gradient aggregation. Comparing with those methods, the proposed method only needs to control a node and one of its neghbours for every step, which leads to low control complexity. The superiority becomes significant for a large scale network.
\item \textbf{Sublinear convergence.} The proposed method successfully achieves $\Ocal{\frac{1}{\sqrt{T}}+\frac{\ln T}{T} + \frac{1}{T}}$ rate of convergence for general convex loss, $\Ocal{\frac{\ln T}{T}+\frac{1}{T}}$ rate of convergence for strongly convex loss, and $\Ocal{\frac{\ln T}{\sqrt{T}} + \frac{1}{\sqrt{T}}}$ rate of convergence for non-convex loss, respectively.  As illustrated in Table \ref{table:convergence:rate}, the proposed method achieves more tight bounds of convergence than existing methods.
\end{itemize}

\section{Theoretical Analysis}
\label{sect:theoretical:analysis}

\subsection{Assumptions}
The following assumptions are used in analysis of convergence, which are widely used in existing studies \cite{sun:2018:mcgd}.
\begin{Assumption}[Markov chain]
\label{assumption:markov:chain}
Assume the Markov chain $\{i_t\}_{t\ge1}$ is time-homogeneous, irreducible, and aperiodic. It has a transition matrix $\P$ and its stationary distribution $\pi_\ast$ with $\pi_\ast = \left [\frac{1}{n}, \frac{1}{n}, \cdots, \frac{1}{n} \right ]$. 
\end{Assumption}

\begin{Assumption}[Smoothness]
\label{assumption:smooth}
Assume $f_v$ with $v\in\Vcal$ is $L$-smooth, that is, for any $\x$ and $\y$, 
\begin{align}
\nonumber
f_v(\x) - f_v(\y) \le \lrangle{\nabla f_v(\y), \x-\y} + \frac{L}{2}\lrnorm{\x-\y}^2.
\end{align}
\end{Assumption}

\begin{Assumption}[Bounded variance]
\label{assumption:bounded:variance}
Assume $\EE_{\a\sim\Dcal_v}\lrnorm{\nabla f_v(\x;\a) - \nabla f_v(\x)}^2 \le \sigma_v^2$, and $\EE_{v}\lrnorm{\nabla f_v(\x) - \nabla f(\x)}^2 \le \sigma_{\Vcal}^2$ for any $v\in\Vcal$. 
\end{Assumption}

\begin{Assumption}[Bounded gradient]
\label{assumption:bounded:gradient}
Assume $\max\left\{\lrnorm{\nabla f_v(\x)}, \lrnorm{\nabla \Phi(\x)} \right\} \le G$ for any $v\in\Vcal$ and $\x$.
\end{Assumption}

\begin{Assumption}[Bounded domain]
\label{assumption:bounded:domain}
Assume $\max \left\{\lrnorm{\x - \y}^2, B_{\Phi}(\x, \y) \right\} \le R^2$ for any $\x$ and $\y$.
\end{Assumption}
After that, we introduce some constants, which will be used in the following analysis. Those constants are first introduced in \cite{sun:2018:mcgd}, and then widely used in existing work \cite{even:2022:mcss:icml}. Suppose the Jordan normal form of transition matrix $\P$ is denoted by 
\begin{align}
\nonumber
\P:=\U \begin{bmatrix}
 1 &  &  & \\ 
 &  \J_2&  & \\ 
 &  & \cdots & \\ 
 &  &  & \J_m
\end{bmatrix} \U^{-1},
\end{align} where $m$ is the number of blocks, and $d_i\ge1$ is the dimension of the $i$-th block submatrix $\J_i$ with $i\in\{2,3,\cdots,m\}$. Additionally, suppose $\rho_i$ represents the $i$-th largest eigenvalue of $\P$, and $\rho$ is denoted by
\begin{align}
\nonumber
\rho := & \frac{\max\{|\rho_2|, |\rho_n|\}+1}{2} \in [0, 1).
\end{align} Specifically, $C_{\P}$ is defined by 
\begin{align}
\label{equa:define:CP}
C_{\P} := \lrincir{\sum_{i=2}^m d_i^2}^{\frac{1}{2}} \lrnorm{\U}_F \lrnorm{\U^{-1}}_F,
\end{align} and $\tau$ is defined by 
\begin{align}
\label{equa:define:tau}
\tau := \max\left\{ \max_{1\le i\le m}\left\{ \left \lceil \frac{2d_i(d_i-1)\lrincir{\log \lrincir{\frac{2d_i}{|\rho_2|\cdot \log\lrincir{\frac{\rho}{\rho_2}}}} -1}}{(d_i+1)\log\lrincir{\frac{\rho}{|\rho_2|}}} \right \rceil \right\}, 0 \right\}.
\end{align} Moreover, useful constants $C_0$-$C_5$ are illustrated in Table \ref{table_constants}.
\begin{table}[!h]
\centering
\caption{Useful constants.}
\renewcommand\arraystretch{2.5}
\begin{tabular}{c|c|c}
\hline
$C_0 :=  \sigma_v^2 + \sigma_{\Vcal}^2 + G^2$ & $C_1 :=  \frac{3LG C_0}{\mu_{\Phi}^2} + \frac{C_0 L^2}{2\mu_{\Phi}^2}$ & $C_2 :=  \frac{3C_0 (L+1)}{2\mu_{\Phi}^2} + 3G^2 C_0$ \\ \hline
$C_3 :=  \frac{3L^2C_0}{2\mu_{\Phi}^2} + \frac{G^2}{2}$ & $C_4 :=  \frac{3G^2}{2}+\sigma_{\Vcal}^2$ & $C_5 :=  \frac{3\lrincir{\sigma_v^2+\sigma_{\Vcal}^2}}{\mu_{\Phi}}$ \\ \hline
\end{tabular}
\label{table_constants}
\end{table}

\subsection{Convergence Rate}

\label{subsect:regret:formulation}
Given a Markov chain $\zeta_T := \{ i_t \}_{t=0}^{T}$ over the network $\Ncal$, the regret of an optimization method $A$ is defined by 
\begin{align}
\label{equa:regret:formulation}
\Rcal_{\zeta_T}^{A}(\w) := \EE \sum_{t=1}^T f_{i_t}(\x_t) - \EE \sum_{t=1}^T  f_{i_t}(\w),
\end{align} where the local loss function $f_{i_t}$ is defined by $f_{i_t}(\y) := \EE_{\a\sim\Dcal_{i_t}} f_{i_t}\lrincir{\y; \a}$. We first bridge connection between the convergence and the regret of an optimization method.

\begin{Theorem}[Connection between convergence and regret]
\label{theorem:bridge:between:convergence:regret}
Given a Markov chain $\zeta_T$, and the sequence of model $\{\x_t\}_{t=1}^T$ produced by an optimization method $A$, we have   
\begin{align}
\nonumber
f(\bar{\x}_T) - f(\x^\ast) \le \frac{\Rcal_{\zeta_T}^A(\x^\ast)}{\lrincir{1 + \lrnorm{\P - \frac{1}{n}\I_n}_{\max} n}T},
\end{align} where $\bar{\x}_T := \frac{1}{T}\sum_{t=1}^T \x_t$, $\I_n:=\1\cdot\1^{\top}$ represents a constant matrix whose element is 1, and $\Rcal_{\zeta_T}^A(\x^\ast)$ is defined in \ref{equa:regret:formulation}. Note that $\lrnorm{\A}_{\max}$ denotes the maximal absolute value of element of $\A$, that is $\lrnorm{\A}_{\max} := \max_{i,j} |[\A]_{i,j}|$.
\end{Theorem}

Theorem \ref{theorem:bridge:between:convergence:regret} shows that the analysis of convergence rate can be reduced to analyze the regret at the optimum of model $\x^\ast$. Therefore, we furthermore focus on the regret analysis of the proposed method, and then achieve its convergence rate directly due to Theorem \ref{theorem:bridge:between:convergence:regret}. 

\subsubsection{Regret Analysis for Convex Loss}
When the loss function $f_v$ of the node $v\in\Vcal$ is convex, its regret is bounded by the following theorem.
\begin{Theorem}[Convex loss]
\label{theorem:regret:convex:analysis}
Suppose $\{\x_t\}_{t=1}^T$ is yielded by Algorithm \ref{algo:marchon}, and choose non-increasing positive sequence $\{\eta_t\}_{t=1}^T$. Let $\hat{\tau} \ge \tau$ for any $t\in[T]$. Under Assumptions \ref{assumption:markov:chain}-\ref{assumption:bounded:domain}, we have
\begin{align}
\nonumber
& \Rcal^{\textsc{MarchOn}}_{\zeta_T}(\x^\ast) \\ \nonumber
\le & C_5 \sum_{t=1}^T \eta_t + \frac{3G^2}{\mu_{\Phi}}\sum_{t=1}^{\hat{\tau}} \eta_t \mathrm{+} \frac{3\lrincir{f(\x_{\hat{\tau}+1}) \mathrm{-} f(\x_{T+1})}}{\mu_{\Phi}} \mathrm{+} \sum_{t=\hat{\tau}+1}^T \lrincir{\frac{3C_1\hat{\tau}}{\mu_{\Phi}}  \sum_{j=t-\hat{\tau}}^{t-1}\eta_j^2 \mathrm{+} \frac{3\lrincir{C_2\mathrm{+}C_3\hat{\tau}}}{\mu_{\Phi}}\eta_t^2 \mathrm{+} \frac{3C_4 C_{\P}\rho^{\hat{\tau}}}{\mu_{\Phi}} \eta_t} \mathrm{+} \frac{R^2}{\eta_T},
\end{align}  when $f_v$ is general convex for any $v\in\Vcal$.
\end{Theorem} Setting appropriate step size $\eta_t$ with $t\in[T]$, we achieve the following bound of regret.
\begin{Corollary}
\label{corollary:regret:convex:analysis}
Suppose $\{\x_t\}_{t=1}^T$ is yielded by Algorithm \ref{algo:marchon}. Let $\hat{\tau} = \tau$, and choose $\eta_t$ by
\begin{align}
\nonumber
\eta_t = \frac{1}{\sqrt{(C_5+C_{\P}\cdot \rho^{\tau})\sqrt{t}}}.
\end{align} Under Assumptions \ref{assumption:markov:chain}-\ref{assumption:bounded:domain}, we have
\begin{align}
\nonumber
\Rcal^{\textsc{MarchOn}}_{\zeta_T}(\x^\ast) \lesssim   \sqrt{C_0+C_{\P} \rho^{\tau}}\cdot \sqrt{T} + \lrincir{\tau^2 + \tau} \cdot \ln T + \sqrt{\tau}, 
\end{align} when $f_v$ is general convex for any $v\in\Vcal$.
\end{Corollary}
\begin{Remark}
Recalling Theorem \ref{theorem:bridge:between:convergence:regret}, and setting the step size $\{\eta_t\}_{t=1}^T$ as in Corollary \ref{corollary:regret:convex:analysis}, we have 
\begin{align}
\nonumber
f(\bar{\x}_T) - f(\x^\ast) \lesssim \Ocal{\frac{1}{\sqrt{T}} + \frac{\ln T}{T} + \frac{1}{T}}
\end{align} for convex $f_v$ with $v\in\Vcal$.  That is, Algorithm \ref{algo:marchon} converges at the rate of $\Ocal{\frac{1}{\sqrt{T}} + \frac{\ln T}{T} + \frac{1}{T}}$ for general convex loss.
\end{Remark}

As illustrated in Table \ref{table:convergence:rate}, the proposed method achieves the best convergence rate. The reason is that \textit{MarchOn} could choose larger step size than those existing methods. Specifically, when $f_v$ with any $v\in\Vcal$ is convex,  \textit{MCGD} chooses $\eta_t \propto \frac{1}{t^q}$ with $\frac{1}{2}<q<1$ \cite{sun:2018:mcgd}, \textit{MarkovSGD} chooses $\eta_t \propto \frac{\ln(\ln(t))\ln^2(t)}{\sqrt{t}}$ \cite{doan:2023:mgd:tac}, \textit{MCSGD} and \textit{EMD} choose $\eta_t \propto \frac{1}{\sqrt{t\ln (t)}}$    \cite{even:2022:mcss:icml,duchi:2012:siamjoo}. However, the proposed method, namely \textit{MarchOn}, chooses $\eta_t \propto \frac{1}{\sqrt{t}}$, which is larger than all of existing methods.

\subsubsection{Regret Analysis for Strongly Convex Loss}
The strongly convex function is usually defined as follows.
\begin{Definition}
\label{definition:strongly:convex:loss}
$f_v$ with $v\in\Vcal$ is $\mu_f$-strongly convex, that is, for any $\x$ and $\y$, 
\begin{align}
\nonumber
f_v(\y) - f_v(\x) \ge \lrangle{\nabla f_v(\x), \y-\x} + \frac{\mu_f}{2}\lrnorm{\x-\y}^2.
\end{align}
\end{Definition} When the loss function $f_v$ with $v\in\Vcal$ is strongly convex, the regret of Algorithm \ref{algo:marchon} is bounded as follows.
\begin{Theorem}[Strongly convex loss]
\label{theorem:regret:strongly:convex:analysis}
Suppose $\{\x_t\}_{t=1}^T$ is yielded by Algorithm \ref{algo:marchon}, and choose non-increasing positive sequence $\{\eta_t\}_{t=1}^T$. Let $\hat{\tau} \ge \tau$ for any $t\in[T]$. Under Assumptions \ref{assumption:markov:chain}-\ref{assumption:bounded:domain}, we have
\begin{align}
\nonumber
& \Rcal^{\textsc{MarchOn}}_{\zeta_T}(\x^\ast) \\ \nonumber
\le & C_5 \sum_{t=1}^T \eta_t + \frac{3G^2}{\mu_{\Phi}}\sum_{t=1}^{\hat{\tau}} \eta_t + \frac{3\lrincir{f(\x_{\hat{\tau}+1}) - f(\x_{T+1})}}{\mu_{\Phi}} + \sum_{t=\hat{\tau}+1}^T \lrincir{\frac{3C_1\hat{\tau}}{\mu_{\Phi}}  \sum_{j=t-\hat{\tau}}^{t-1}\eta_j^2 + \frac{3\lrincir{C_2+C_3\hat{\tau}}}{\mu_{\Phi}}\eta_t^2 + \frac{3C_4 C_{\P}\rho^{\hat{\tau}}}{\mu_{\Phi}} \eta_t} \\ \nonumber 
& + \frac{\mu_{\Phi} R^2}{2} \sum_{t=2}^T \lrincir{\frac{1}{\eta_t} - \frac{1}{\eta_{t-1}} - \frac{\mu_f}{\mu_{\Phi}}} + \frac{R^2}{\eta_1},
\end{align}  when $f_v$ is $\mu_f$-strongly convex for any $v\in\Vcal$.
\end{Theorem} Setting appropriate step size $\eta_t$ with $t\in[T]$, we achieve the following bound of regret.
\begin{Corollary}
\label{corollary:regret:strongly:convex:analysis}
Suppose $\{\x_t\}_{t=1}^T$ is yielded by Algorithm \ref{algo:marchon}. Let $\hat{\tau} = \tau$, and choose $\eta_t$ by
\begin{align}
\nonumber
\eta_t = \min\left\{\sqrt{C_5+C_{\P}\cdot \rho^{\tau}}, \frac{\mu_{\Phi}}{\mu_f}\right\}\cdot \frac{1}{t}.
\end{align} Under Assumptions \ref{assumption:markov:chain}-\ref{assumption:bounded:domain}, we have
\begin{align}
\nonumber
\Rcal^{\textsc{MarchOn}}_{\zeta_T}(\x^\ast) \lesssim \sqrt{C_0+C_{\P} \rho^{\tau}}\cdot \ln T + \tau^2 + \tau + \ln \tau,
\end{align}  when $f_v$ is general convex for any $v\in\Vcal$.
\end{Corollary}
\begin{Remark}
Recalling Theorem \ref{theorem:bridge:between:convergence:regret}, and setting the step size $\{\eta_t\}_{t=1}^T$ as in Corollary \ref{corollary:regret:strongly:convex:analysis}, we have 
\begin{align}
\nonumber
f(\bar{\x}_T) - f(\x^\ast) \lesssim \Ocal{\frac{\ln T}{T} + \frac{1}{T}}\end{align} for strongly convex $f_v$ with $v\in\Vcal$.  That is, Algorithm \ref{algo:marchon} converges at the rate of $\Ocal{\frac{\ln T}{T} + \frac{1}{T}}$ for strongly convex loss.
\end{Remark}

\subsubsection{Regret Analysis for Non-convex Loss}

\begin{Theorem}[Non-convex loss] 
\label{theorem:regret:nonconvex:analysis}
Suppose $\{\x_t\}_{t=1}^T$ is yielded by Algorithm \ref{algo:marchon}, choose $\eta_t = \eta>0$ for any $t\in[T]$, and let $\hat{\tau} \ge \tau$ for any $t\in[T]$. Under Assumptions \ref{assumption:markov:chain}-\ref{assumption:bounded:domain}, we have
\begin{align}
\nonumber
\EE \sum_{t=1}^T \lrnorm{\nabla f(\x_t)}^2  \le G^2\hat{\tau} + \frac{f(\x_{\hat{\tau}+1}) - f(\x_{T+1})}{\eta} + \lrincir{C_1 \hat{\tau}^2 \eta + \lrincir{C_2+C_3\hat{\tau}}\eta + C_4 C_{\P} \rho^{\hat{\tau}}}(T - \hat{\tau}).
\end{align} 
\end{Theorem} Setting appropriate step size $\eta_t$ with $t\in[T]$, we achieve the following bound of regret.
\begin{Corollary}
\label{corollary:regret:strongly:convex:analysis}
Suppose $\{\x_t\}_{t=1}^T$ is yielded by Algorithm \ref{algo:marchon}. Let 
\begin{align}
\nonumber
\hat{\tau} = \max\left\{\tau, \frac{\ln T}{2\ln \lrincir{1/\rho}}\right\},
\end{align} and choose $\eta_t$ by
\begin{align}
\nonumber
\eta_t = \eta = \frac{2\ln \lrincir{1/\rho}}{\sqrt{C_1 T}\ln T}.
\end{align} Under Assumptions \ref{assumption:markov:chain}-\ref{assumption:bounded:domain}, we have
\begin{align}
\nonumber
\EE \frac{1}{T}\sum_{t=1}^T \lrnorm{\nabla f(\x_t)}^2 \lesssim \sqrt{\sigma_v^2+\sigma_{\Vcal}^2} \frac{\ln T}{\sqrt{T}} + \frac{\sqrt{\sigma_v^2+\sigma_{\Vcal}^2} + C_{\P}}{\sqrt{T}}.
\end{align}  when $f_v$ is general non-convex for any $v\in\Vcal$.
\end{Corollary}
%
%
%
%
As we can see, Algorithm \ref{algo:marchon} achieves convergence at the rate of $\Ocal{\frac{\ln T}{\sqrt{T}} + \frac{1}{\sqrt{T}}}$ for non-convex loss. 

\section{Empirical Studies}
\label{sect:empirical:study}

\subsection{Experimental Settings}
We conduct logistic regression for binary classification in empirical studies, in which $f_v$ is denoted by 
\begin{align}
\nonumber
f_v(\x; \Dcal_v) = \frac{1}{n_v} \sum_{i=1}^{n_v} \lrincir{1+\exp\lrincir{-y\a^{\top}\x}}.
\end{align} Here, $\Dcal_v$ represents the local data of the node $v$. $n_v = |\Dcal_v|$ represents its cardinality. $\a\in\RR^d$ and $y\in\{-1, 1\}$ represents one instance and its corresponding label, respectively. In the experiment, we choose four Libsvm datasets, namely cod-rna\footnote{\url{https://www.csie.ntu.edu.tw/~cjlin/libsvmtools/datasets/binary.html\#cod-rna}}, covtype\footnote{\url{https://www.csie.ntu.edu.tw/~cjlin/libsvmtools/datasets/binary.html\#covtype}}, ijcnn1\footnote{\url{https://www.csie.ntu.edu.tw/~cjlin/libsvmtools/datasets/binary.html\#ijcnn1}}, and phishing\footnote{\url{https://www.csie.ntu.edu.tw/~cjlin/libsvmtools/datasets/binary.html\#phishing}}. All samples of those datasets are normalized to $[-1, 1]$. Statistics of those datasets are summarized in Table \ref{table:datasets}. 

\begin{table*}
\centering
\caption{Statistics of datasets.}
\renewcommand\arraystretch{2}
\begin{tabular}{c|c|c|c|c}
\hline
Datasets           & cod-rna & covtype & ijcnn1 & phishing  \\ \hline \hline
Number of features &    $8$     &    $54$     &    $22$    &     $68$      \\ \hline
Number of samples  &    $59,535$     &    $581,012$     &    $49,990$    &  $11,055$  \\ \hline      
\end{tabular}
\label{table:datasets}
\end{table*}

We compare the proposed method, that is \textit{MarchOn}, with three other existing methods, namely \textit{MCGD} \cite{sun:2018:mcgd}, \textit{MarkovSGD} \cite{doan:2023:mgd:tac}, and \textit{MCSGD} \cite{even:2022:mcss:icml}. Those existing methods generally update models by iteratively performing $\x_{t+1} = \x_t - \eta_t \nabla f_v(\x_t; \a\sim\Dcal_v)$ with $v\in\Vcal$ for unconstrained optimization problems. As we have shown, major difference of those methods is taking different strategies to choose step size. In the experiment, we choose $\eta_t \propto \frac{1}{t^q}$ with $\frac{1}{2}<q<1$ for \textit{MCGD}, $\eta_t \propto \frac{\ln(\ln(t))\ln^2(t)}{\sqrt{t}}$ for \textit{MarkovSGD}, $\eta_t \propto \frac{1}{\sqrt{t\ln (t)}}$ for \textit{MCSGD}, and $\eta_t \propto \frac{1}{\sqrt{t}}$ for \textit{MarchOn}. All those methods are implemented by using Python 3.11. Those methods are run by using Intel I7 CPU with $12$ cores.

\begin{figure*}
\setlength{\abovecaptionskip}{0pt}
\setlength{\belowcaptionskip}{0pt}
\centering 
\subfigure[cod-rna]{\includegraphics[width=0.24\columnwidth]{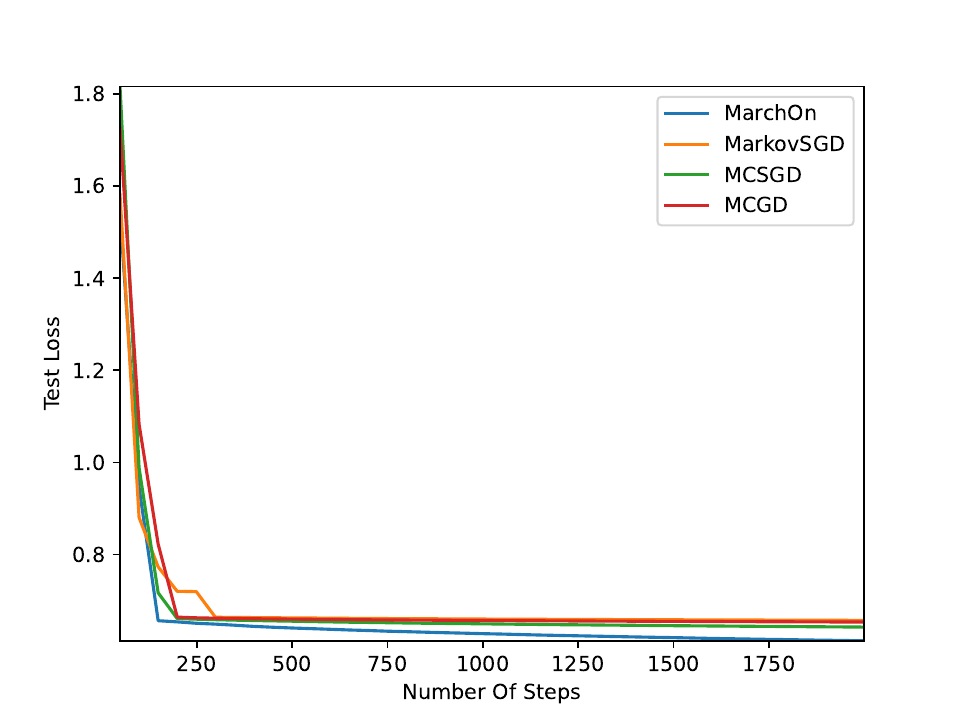}\label{figs_convergence_codrna}}
\subfigure[covtype]{\includegraphics[width=0.24\columnwidth]{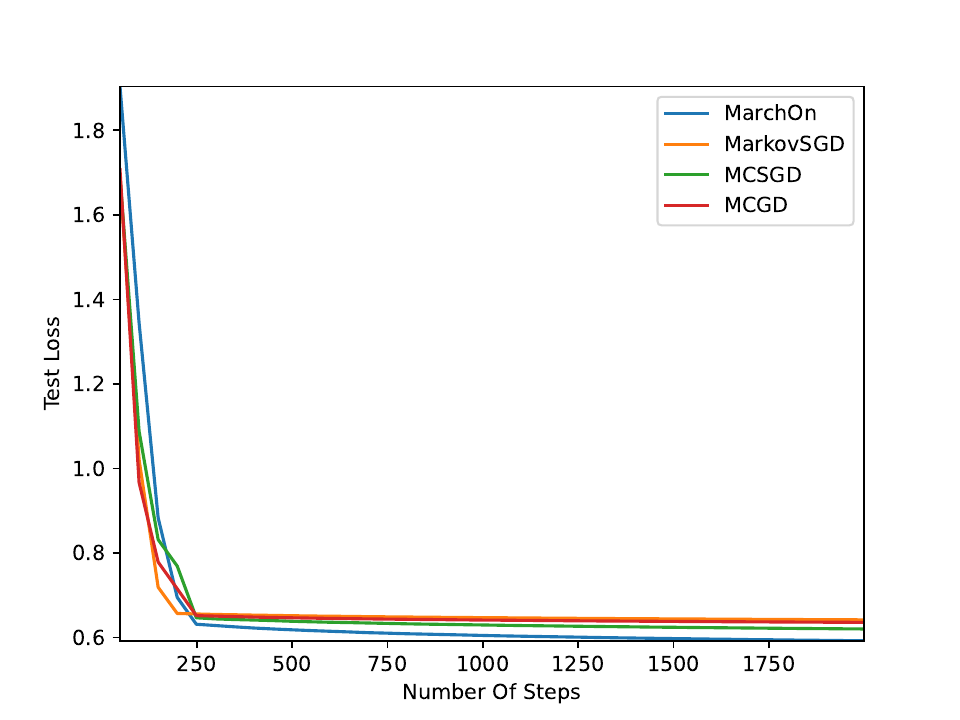}\label{figs_convergence_covtype}}
\subfigure[ijcnn1]{\includegraphics[width=0.24\columnwidth]{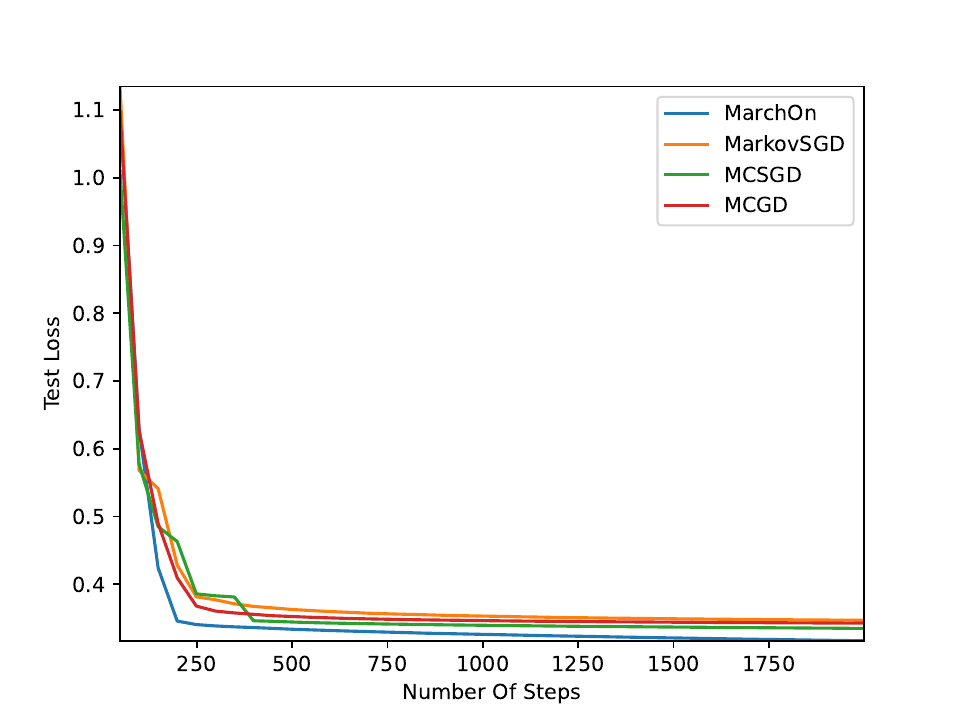}\label{figs_convergence_ijcnn1}}
\subfigure[phishing]{\includegraphics[width=0.24\columnwidth]{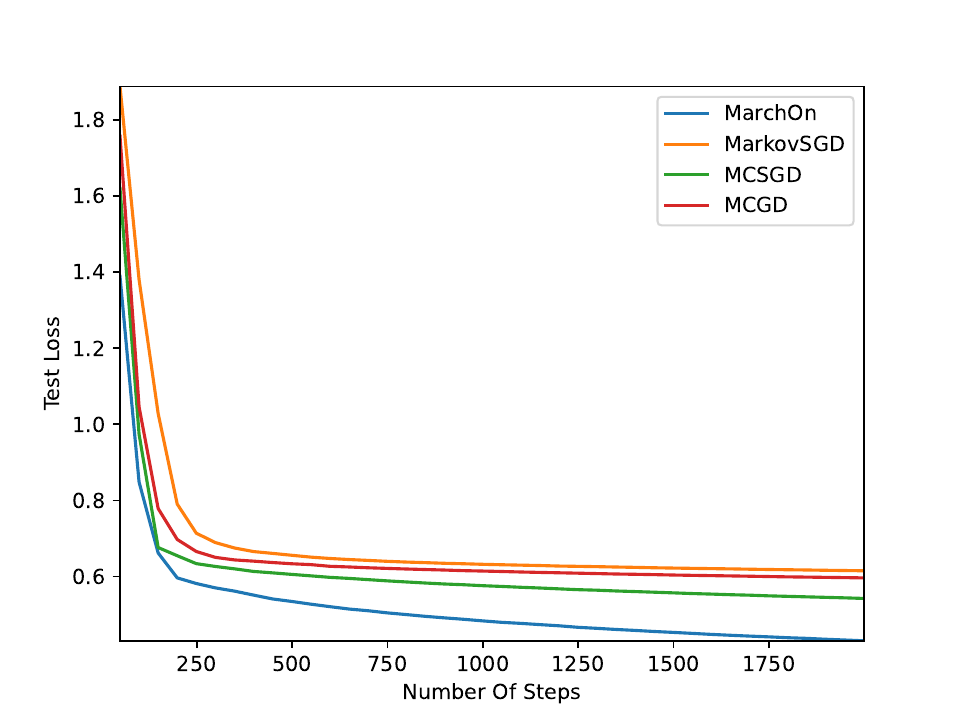}\label{figs_convergence_phishing}}
\caption{Comparison of convergence on a data federation for different methods. The network consists of $50$ nodes, and its topology is a complete graph. }
\label{figure_convergence_50nodes}
\end{figure*}
\begin{figure*}
\setlength{\abovecaptionskip}{0pt}
\setlength{\belowcaptionskip}{0pt}
\centering 
\subfigure[Cod-rna]{\includegraphics[width=0.24\columnwidth]{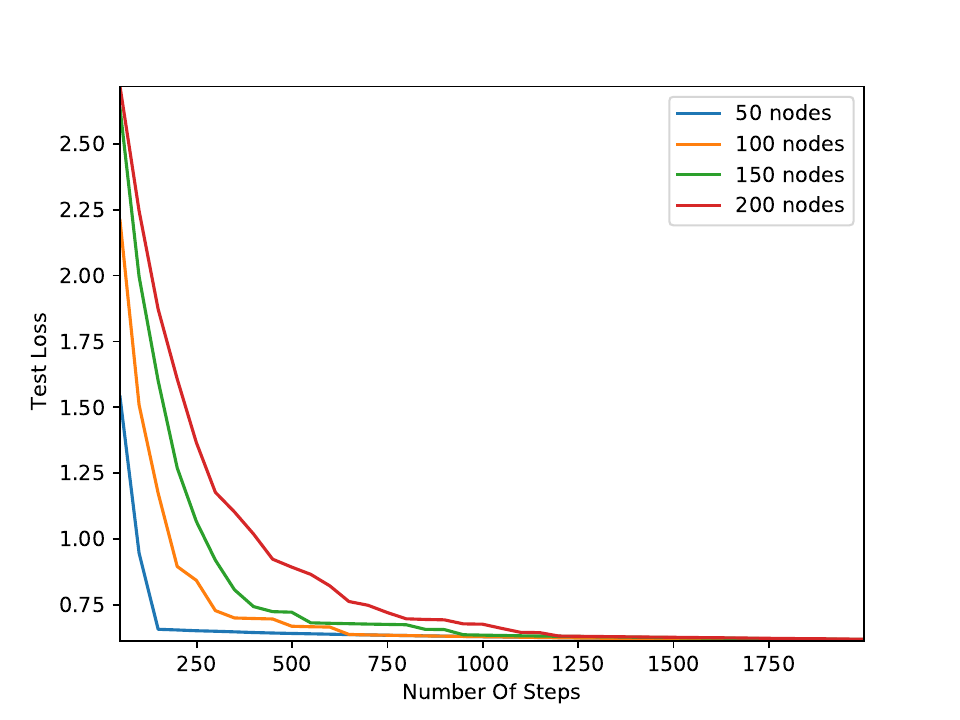}\label{figs_networksize_codrna}}
\subfigure[covtype]{\includegraphics[width=0.24\columnwidth]{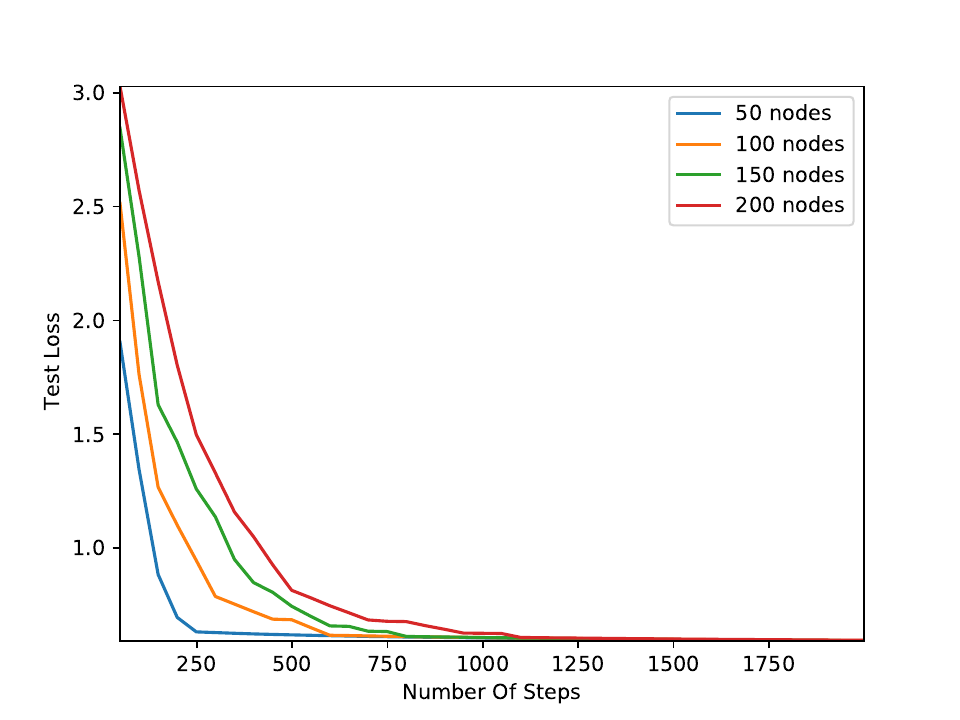}\label{figs_networksize_covtype}}
\subfigure[ijcnn1]{\includegraphics[width=0.24\columnwidth]{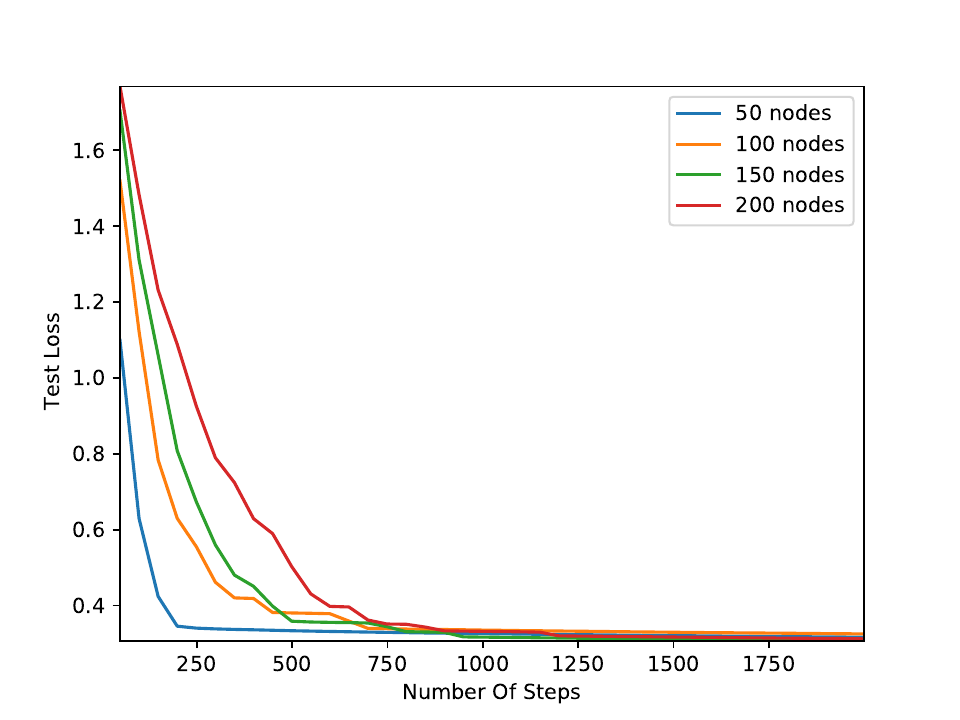}\label{figs_networksize_ijcnn1}}
\subfigure[phishing]{\includegraphics[width=0.24\columnwidth]{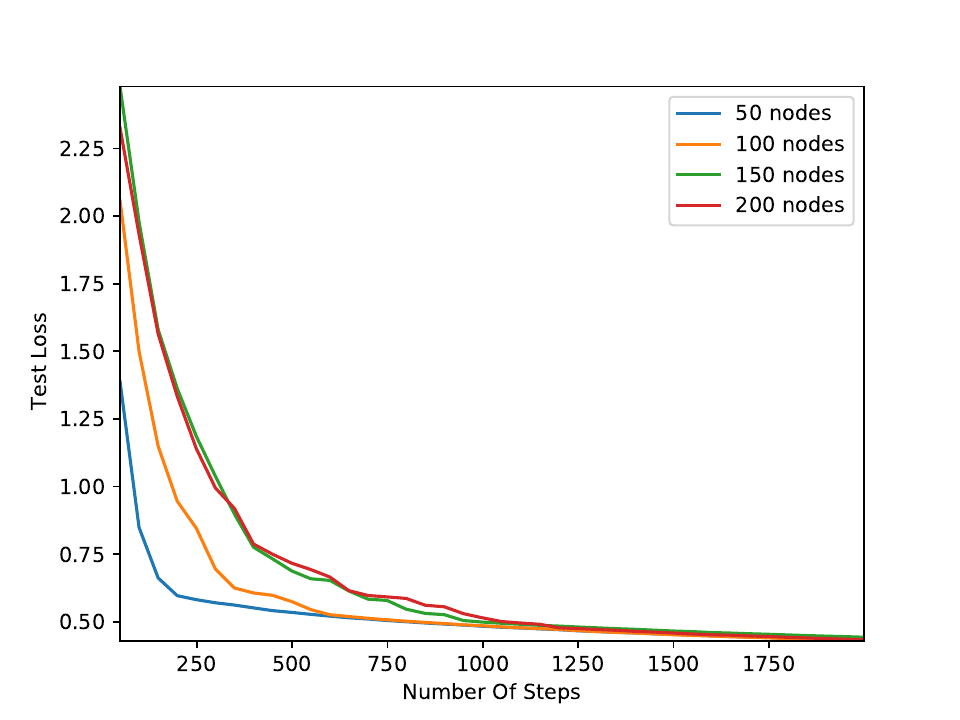}\label{figs_networksize_phishing}}
\caption{Comparison of convergence on a data federation by varying size of network. The network has a topology of the complete graph.}
\label{figure_networksize_completegraph}
\end{figure*}
\begin{figure*}
\setlength{\abovecaptionskip}{0pt}
\setlength{\belowcaptionskip}{0pt}
\centering 
\subfigure[Cod-rna]{\includegraphics[width=0.24\columnwidth]{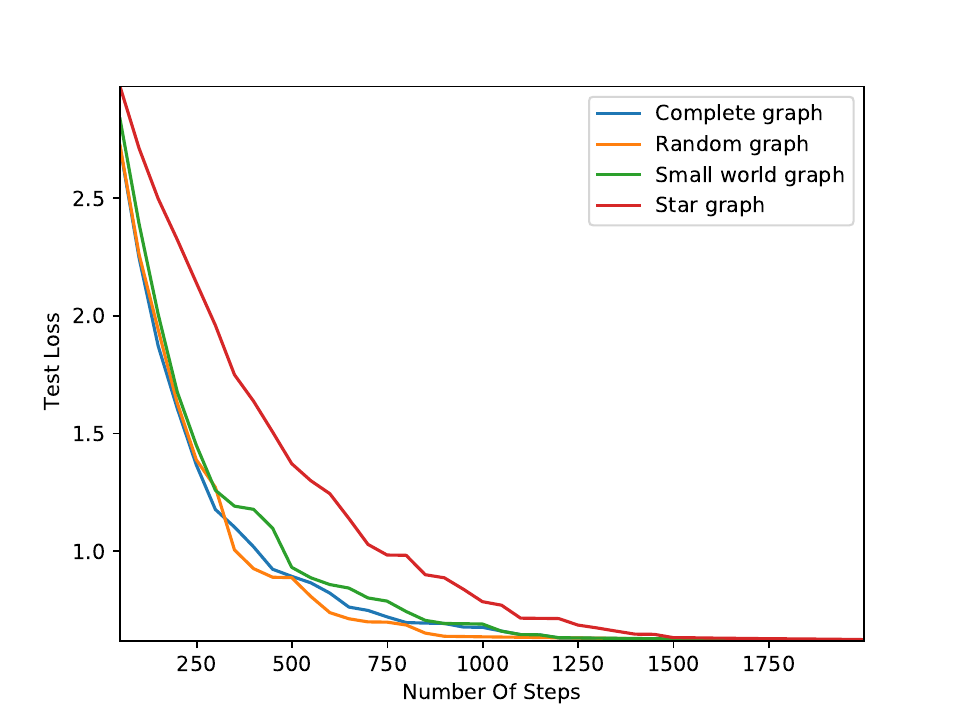}\label{figs_networktopology_codrna}}
\subfigure[covtype]{\includegraphics[width=0.24\columnwidth]{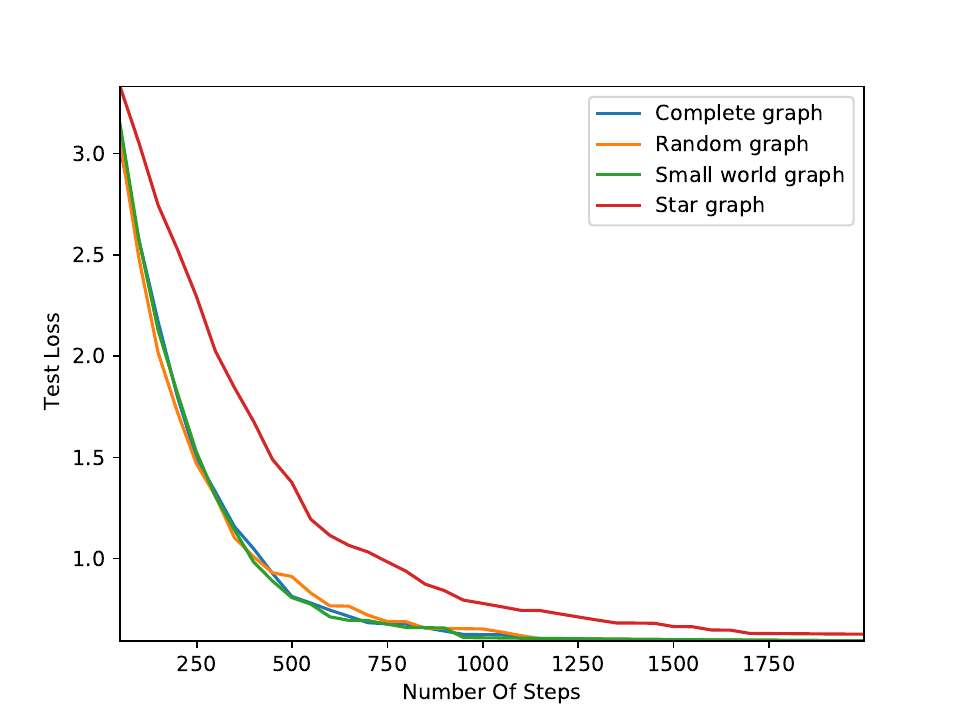}\label{figs_networktopology_covtype}}
\subfigure[ijcnn1]{\includegraphics[width=0.24\columnwidth]{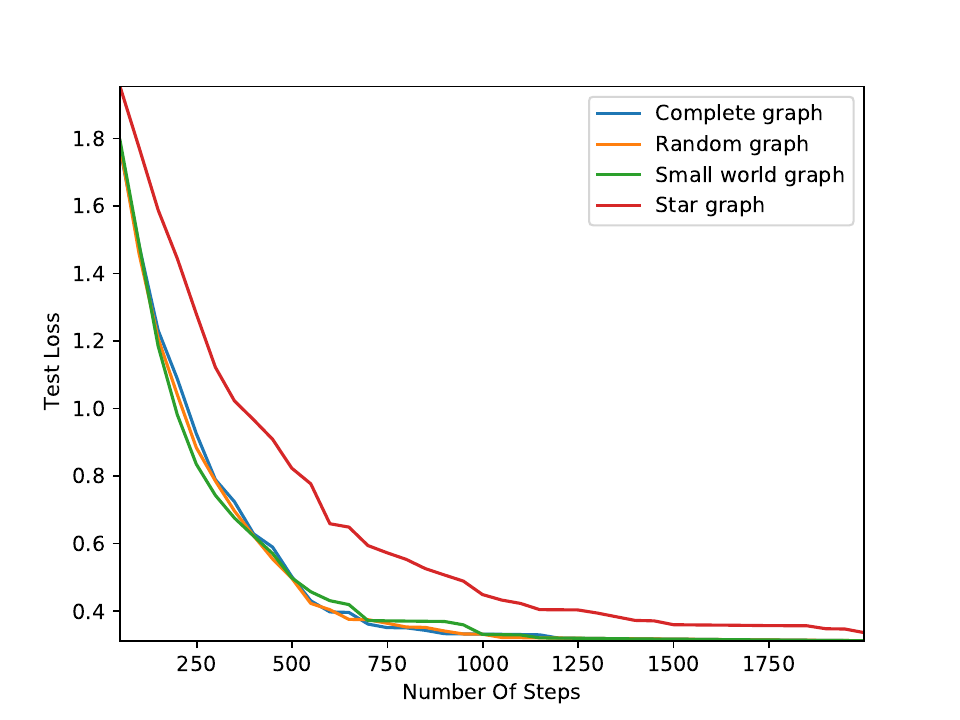}\label{figs_networktopology_ijcnn1}}
\subfigure[phishing]{\includegraphics[width=0.245\columnwidth]{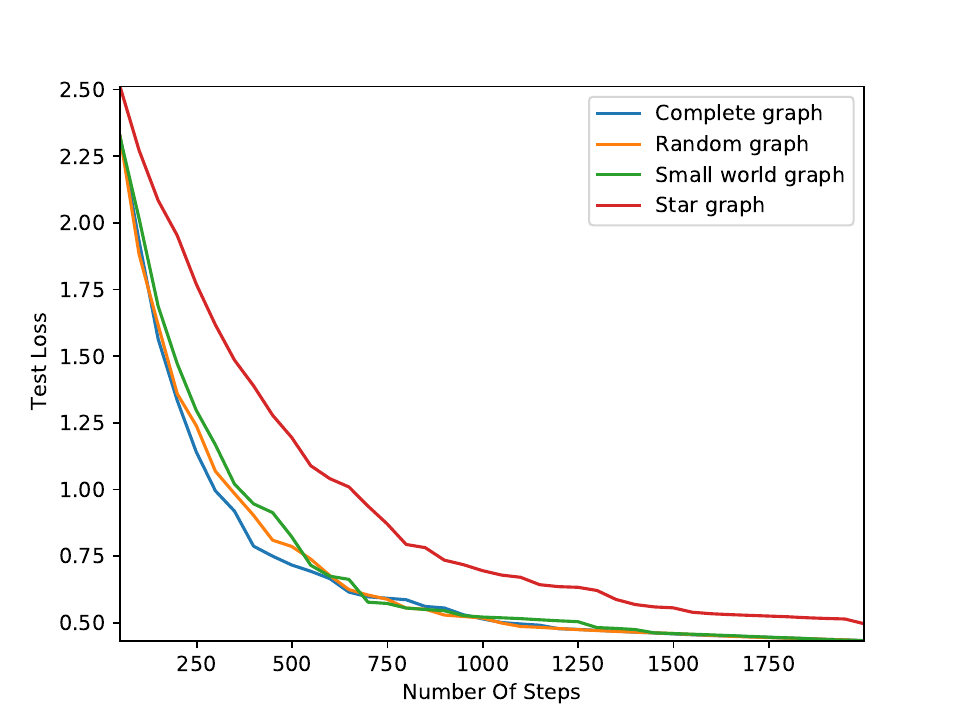}\label{figs_networktopology_phishing}}
\caption{Comparison of convergence on a data federation by varying topology of network. The network consists of $200$ nodes.}
\label{figure_convergence_networktopology}
\end{figure*}

\subsection{Numerical Results}

\textit{The proposed method outperforms others by obtaining best rate of convergence.} As illustrated in Figure \ref{figure_convergence_50nodes}, all of those methods converge fast in a network owning the topology of complete graph, but the proposed method, that is \textit{MarchOn}, achieves better performance than others. It validates the theoretical analysis, that is, the proposed method is able to achieve tighter bound of convergence rate. Specifically, we observe that the superiority becomes significant after $250$ steps, and the proposed method gains much more benefits for a large number of steps.

\textit{The proposed method gains benefits of performance with the reduction of network's size.} In the case, the network has the topology of the complete graph. As illustrated in Figure \ref{figure_networksize_completegraph}, we evaluate the robustness of the proposed method, that is \textit{MarchOn}, by varying the size of network. As we can see, \textit{MarchOn} converges fast in a large network, and performs better with the decrease of the size of network. Since $\rho = \frac{1}{n}$ holds for a network with topology of the complete graph, the observation consists with the theoretical analysis, that is, the convergence becomes large with the increase of $\rho$. 

\textit{The proposed method converges fast, and its performance is robust to the network's topology.}  As illustrated in Figure \ref{figure_convergence_networktopology}, we evaluate the robustness of the proposed method, that is \textit{MarchOn}, by varying the topology of network. In the experiment, every network owns $200$ nodes. The network owning the topology of the random graph is Erd\"{o}s-R\'{e}nyi random network, which is  generated by using a well-known Python package: \textit{networkx}\footnote{\url{https://networkx.org}}. The network owning the topology of the small world graph is generated by using the same way. As we can see,  \textit{MarchOn} achieves similar performance for the network with the topology of complete graph and other networks with the topology of either random graph or small world graph. Although the network owning the star graph converges slowly, its performance is comparable with others. Especially, considering that other networks has $\Ocal{n^2}$ edges, and the network owning the star graph has $\Ocal{n}$ edges, such gap of convergence is not much significant.

\section{Conclusion}
\label{sect:conclusion}
As a general optimization method, mirror descent with Markov chain stochastic gradient is investigated in the setting of federated learning. First, a new method named \textit{MarchOn} is proposed. It allows model to travel over data federation, and updates model by using local data of every visited node. Second, a new framework of theoretical analysis is developed. It successfully provides sub-linear rate of convergence for convex, strongly convex, and non-convex loss. Finally, extensive empirical studies are conducted to evaluate both effectiveness and robustness of the proposed method. Comparing with existing methods, the proposed method achieves the best performance on convergence, and performs robustness with respect to the network's size and topology. However, there are still some limitations. For example, some nodes in the data federation may be malicious.  The model may not believe in every node, and update parameters by using local data. In the future, the work will be extended to the setting of federated learning with malicious nodes.

\section*{Appendix}

\subsection{Proof of Theorem \ref{theorem:bridge:between:convergence:regret}}
\textbf{Theorem} \ref{theorem:bridge:between:convergence:regret}. Given a Markov chain $\zeta_T$, and the sequence of model $\{\x_t\}_{t=1}^T$ produced by an optimization method $A$, we have   
\begin{align}
\nonumber
f(\bar{\x}_T) - f(\x^\ast) \le \frac{\Rcal_{\zeta_T}^A(\x^\ast)}{\lrincir{1 + \lrnorm{\P - \frac{1}{n}\I_n}_{\max} n}T},
\end{align} where $\bar{\x}_T := \frac{1}{T}\sum_{t=1}^T \x_t$, and $\Rcal_{\zeta_T}^A(\x^\ast)$ is defined in \ref{equa:regret:formulation}.
\begin{proof}
\begin{align}
\nonumber
& \EE_{i_t} \lrincir{f_{i_t}(\x_t) - f_{i_t}(\x^\ast)} \\ \nonumber
= & \sum_{i=1}^n \lrincir{f_i(\x_t) - f_i(\x^\ast)} \cdot \PP\lrincir{i_t = i | \chi_t} \\ \nonumber
\overset{\textcircled{1}}{=} & \sum_{i=1}^n \lrincir{f_i(\x_t) - f_i(\x^\ast)} \cdot \PP\lrincir{i_t = i | i_{t-1}} \\ \nonumber
= & \sum_{i=1}^n \lrincir{f_i(\x_t) - f_i(\x^\ast)} \cdot \PP\lrincir{i_t = i | i_{t-1}} - \lrincir{f(\x_t) - f(\x^\ast)} + \lrincir{f(\x_t) - f(\x^\ast)} \\ \nonumber
= & \sum_{i=1}^n \lrincir{f_i(\x_t) - f_i(\x^\ast)} \cdot \PP\lrincir{i_t = i | i_{t-1}} - \sum_{i=1}^n \frac{1}{n}\lrincir{f_i(\x_t) - f_i(\x^\ast)} + \lrincir{f(\x_t) - f(\x^\ast)} \\ \nonumber
= & \sum_{i=1}^n \lrincir{\PP\lrincir{i_t = i | i_{t-1}} - \frac{1}{n}} \cdot \lrincir{f_i(\x_t) - f_i(\x^\ast)} + f(\x_t) - f(\x^\ast).
\end{align} $\textcircled{1}$ holds due to the Markov property. Rearranging items, we obtain
\begin{align}
\nonumber
& f(\x_t) - f(\x^\ast) \\ \nonumber
= & \EE_{i_t} \lrincir{f_{i_t}(\x_t) - f_{i_t}(\x^\ast)} - \sum_{i=1}^n \lrincir{\PP\lrincir{i_t = i | i_{t-1}} - \frac{1}{n}} \cdot \lrincir{f_i(\x_t) - f_i(\x^\ast)} \\ \nonumber
= & \EE_{i_t} \lrincir{f_{i_t}(\x_t) - f_{i_t}(\x^\ast)} - \sum_{i=1}^n \lrincir{[\P]_{i_{t-1}, i_t} - \frac{1}{n}} \cdot \lrincir{f_i(\x_t) - f_i(\x^\ast)} \\ \nonumber
\overset{\textcircled{1}}{\le} & \EE_{i_t} \lrincir{f_{i_t}(\x_t) - f_{i_t}(\x^\ast)} -  \lrnorm{\P - \frac{1}{n}\I_n}_{\max} \cdot \sum_{i=1}^n \lrincir{f_i(\x_t) - f_i(\x^\ast)} \\ \nonumber
= & \EE_{i_t} \lrincir{f_{i_t}(\x_t) - f_{i_t}(\x^\ast)} -  \lrnorm{\P - \frac{1}{n}\I_n}_{\max} n \cdot \lrincir{f(\x_t) - f(\x^\ast)}.
\end{align} $\textcircled{1}$ holds due to $\lrnorm{\A}_{\max} := \max_{i,j} |[\A]_{i,j}|$, that is the maximum of absolute values of the given matrix $\A$. Furthermore, we obtain
\begin{align}
\nonumber
f(\x_t) - f(\x^\ast) \le \frac{1}{1 + \lrnorm{\P - \frac{1}{n}\I_n}_{\max} n} \cdot \EE_{i_t} \lrincir{f_{i_t}(\x_t) - f_{i_t}(\x^\ast)}.
\end{align} Therefore, we obtain
\begin{align}
\nonumber
f(\bar{\x}_T) - f(\x^\ast) \le \frac{1}{T}\sum_{t=1}^T f(\x_t) - f(\x^\ast) \le \frac{\Rcal_{\zeta_T}^A(\x^\ast)}{\lrincir{1 + \lrnorm{\P - \frac{1}{n}\I_n}_{\max} n}T} \\ \nonumber
\end{align} It finally completes the proof.
\end{proof}

\subsection{Proof of Theorem \ref{theorem:regret:convex:analysis}}
\textbf{Theorem} \ref{theorem:regret:convex:analysis}. Suppose $\{\x_t\}_{t=1}^T$ is yielded by Algorithm \ref{algo:marchon}, choose non-increasing positive sequence $\{\eta_t\}_{t=1}^T$. Let $\hat{\tau} \ge \tau$ for any $t\in[T]$. Under Assumptions \ref{assumption:markov:chain}-\ref{assumption:bounded:domain}, we have
\begin{align}
\nonumber
& \Rcal^{\textsc{MarchOn}}_{\zeta_T}(\x^\ast) \\ \nonumber
\le & C_5 \sum_{t=1}^T \eta_t + \frac{3G^2}{\mu_{\Phi}}\sum_{t=1}^{\hat{\tau}} \eta_t + \frac{3\lrincir{f(\x_{\hat{\tau}+1}) - f(\x_{T+1})}}{\mu_{\Phi}} + \sum_{t=\hat{\tau}+1}^T \lrincir{\frac{3C_1\hat{\tau}}{\mu_{\Phi}}  \sum_{j=t-\hat{\tau}}^{t-1}\eta_j^2 + \frac{3\lrincir{C_2+C_3\hat{\tau}}}{\mu_{\Phi}}\eta_t^2 + \frac{3C_4 C_{\P}\rho^{\hat{\tau}}}{\mu_{\Phi}} \eta_t} + \frac{R^2}{\eta_T},
\end{align}  when $f_v$ is general convex for any $v\in\Vcal$.
\begin{proof}
Given any $i_t\in[n]$, we have 
\begin{align}
\nonumber
& f_{i_t}(\x_t) -f_{i_t}(\x^\ast) \\ \nonumber
\le & \lrangle{\nabla f_{i_t}(\x_t), \x_t - \x^\ast } \\ \nonumber
= & \EE_{\a\sim\Dcal_{i_t}}\lrangle{\nabla f_{i_t}(\x_t;\a), \x_t - \x^\ast } \\ \nonumber
= & \EE_{\a\sim\Dcal_{i_t}} \lrangle{\nabla f_{i_t}(\x_t;\a), \x_t - \x_{t+1}} + \EE_{\a\sim\Dcal_{i_t}} \lrangle{\nabla f_{i_t}(\x_t;\a), \x_{t+1} - \x^\ast } \\ \nonumber
\overset{\textcircled{1}}{\le} & \EE_{\a\sim\Dcal_{i_t}} \lrangle{\nabla f_{i_t}(\x_t;\a), \x_t - \x_{t+1}} + \EE_{\a\sim\Dcal_{i_t}} \frac{B_{\Phi}(\x^\ast, \x_t) -  B_{\Phi}(\x^\ast, \x_{t+1}) - B_{\Phi}(\x_{t+1}, \x_t)}{\eta_t} \\ \nonumber
\overset{\textcircled{2}}{\le} & \EE_{\a\sim\Dcal_{i_t}} \lrangle{\nabla f_{i_t}(\x_t;\a), \x_t - \x_{t+1}}  - \frac{\mu_{\Phi}}{2\eta_t} \EE_{\a\sim\Dcal_{i_t}} \lrnorm{\x_{t+1} - \x_t}^2 + \EE_{\a\sim\Dcal_{i_t}} \frac{B_{\Phi}(\x^\ast, \x_t) -  B_{\Phi}(\x^\ast, \x_{t+1})}{\eta_t}  \\ \nonumber
\overset{\textcircled{3}}{\le} & \frac{\eta_t}{\mu_{\Phi}}\cdot\EE_{\a\sim\Dcal_{i_t}} \lrnorm{\nabla f_{i_t}(\x_t;\a)}^2 -\frac{\mu_{\Phi}}{4\eta_t}\cdot \EE_{\a\sim\Dcal_{i_t}} \lrnorm{\x_{t+1}-\x_t}^2 + \EE_{\a\sim\Dcal_{i_t}} \frac{B_{\Phi}(\x^\ast, \x_t) -  B_{\Phi}(\x^\ast, \x_{t+1})}{\eta_t} \\ \nonumber
\le & \frac{\eta_t}{\mu_{\Phi}}\cdot \EE_{\a\sim\Dcal_{i_t}} \lrnorm{\nabla f_{i_t}(\x_t;\a)}^2 + \EE_{\a\sim\Dcal_{i_t}} \frac{B_{\Phi}(\x^\ast, \x_t) -  B_{\Phi}(\x^\ast, \x_{t+1})}{\eta_t}.
\end{align} $\textcircled{1}$ holds due to Lemma \ref{lemma_mirror_descent_update_rule} by setting $\g = \nabla f_{i_t}(\x_t)$, $\u_t = \x_t$, $\u_{t+1} = \x_{t+1}$, $\u^\ast = \x^\ast$, and $\lambda = \eta_t$. $\textcircled{2}$ holds due to $\Phi$ is $\mu_{\Phi}$-strongly convex.  $\textcircled{3}$  holds because $\lrangle{\u,\v} \le \frac{a}{2} \lrnorm{\u}^2 + \frac{1}{2a}\lrnorm{\v}^2$ holds for any $\u$, $\v$, and $a>0$. 

Telescoping it over $t$, we have
\begin{align}
\nonumber
&\Rcal^{\textsc{MarchOn}}_{\zeta_T}(\x^\ast) = \EE \sum_{t=1}^T \lrincir{f_{i_t}(\x_t) -f_{i_t}(\x^\ast)}  \\ \nonumber
\le & \sum_{t=1}^T\frac{\eta_t}{\mu_{\Phi}}\EE_{\a\sim\Dcal_{i_t}} \lrnorm{\nabla f_{i_t}(\x_t;\a)}^2+ \sum_{t=1}^T \EE_{\a\sim\Dcal_{i_t}} \frac{ B_{\Phi}(\x^\ast, \x_t) -  B_{\Phi}(\x^\ast, \x_{t+1}) }{\eta_t} \\ \nonumber 
\le & \sum_{t=1}^T\frac{\eta_t}{\mu_{\Phi}}\EE_{\a\sim\Dcal_{i_t}} \lrnorm{\nabla f_{i_t}(\x_t;\a)}^2+ \sum_{t=2}^T \lrincir{\frac{B_{\Phi}(\x^\ast, \x_t)}{\eta_t}  -  \frac{B_{\Phi}(\x^\ast, \x_t)}{\eta_{t-1}}} + \frac{B_{\Phi}(\x^\ast, \x_1)}{\eta_1} - \frac{B_{\Phi}(\x^\ast, \x_{T+1})}{\eta_T}\\ \nonumber 
\le & \sum_{t=1}^T\frac{\eta_t}{\mu_{\Phi}}\EE_{\a\sim\Dcal_{i_t}} \lrnorm{\nabla f_{i_t}(\x_t;\a)}^2 + \sum_{t=2}^T B_{\Phi}(\x^\ast, \x_t)\lrincir{\frac{1}{\eta_t} - \frac{1}{\eta_{t-1}}} + \frac{B_{\Phi}(\x^\ast, \x_1)}{\eta_1}\\ \nonumber 
\le & \sum_{t=1}^T\frac{\eta_t}{\mu_{\Phi}}\EE_{\a\sim\Dcal_{i_t}} \lrnorm{\nabla f_{i_t}(\x_t;\a)}^2 + \sum_{t=2}^T R^2\lrincir{\frac{1}{\eta_t} - \frac{1}{\eta_{t-1}}} + \frac{R^2}{\eta_1}\\ \label{equa:regret:marchon:proof:convex} 
= & \underbrace{\sum_{t=1}^T\frac{\eta_t}{\mu_{\Phi}}\EE_{\a\sim\Dcal_{i_t}} \lrnorm{\nabla f_{i_t}(\x_t;\a)}^2}_{I_1} + \frac{R^2}{\eta_T}.
\end{align} Focusing on $I_1$, we obtain
\begin{align}
\nonumber
& \EE I_1 = \EE \sum_{t=1}^T\frac{\eta_t}{\mu_{\Phi}}\lrnorm{\nabla f_{i_t}(\x_t;\a) - \nabla f_{i_t}(\x_t) + \nabla f_{i_t}(\x_t) - \nabla f(\x_t) + \nabla f(\x_t)}^2 \\ \nonumber
\le & \EE \sum_{t=1}^T\frac{3\eta_t \lrnorm{\nabla f_{i_t}(\x_t;\a) - \nabla f_{i_t}(\x_t)}^2}{\mu_{\Phi}} + \EE \sum_{t=1}^T\frac{3\eta_t \lrnorm{\nabla f_{i_t}(\x_t) - \nabla f(\x_t)}^2}{\mu_{\Phi}} + \EE \sum_{t=1}^T\frac{3\eta_t \lrnorm{\nabla f(\x_t)}^2}{\mu_{\Phi}}\\ \nonumber
\le & \frac{3\sigma_v^2}{\mu_{\Phi}} \sum_{t=1}^T \eta_t + \frac{3\sigma_{\Vcal}^2}{\mu_{\Phi}} \sum_{t=1}^T \eta_t + \EE \sum_{t=1}^T\frac{3\eta_t \lrnorm{\nabla f(\x_t)}^2}{\mu_{\Phi}} \\ \nonumber
\overset{\textcircled{1}}{\le} & \frac{3\sigma_v^2+3\sigma_{\Vcal}^2}{\mu_{\Phi}} \sum_{t=1}^T \eta_t +  LG \sum_{t=\hat{\tau}+1}^T \eta_t^2 + \frac{3LG \lrincir{\sigma_v^2 + \sigma_{\Vcal}^2 + G^2} \hat{\tau}}{\mu_{\Phi}^2} \sum_{t=\hat{\tau}+1}^T \sum_{j=t-\hat{\tau}}^{t-1}\eta_j^2 \\ \nonumber
& + \lrincir{\sigma_v^2 + \sigma_{\Vcal}^2+G^2} \lrincir{ \frac{L^2\hat{\tau}\sum_{t=\hat{\tau}+1}^T \sum_{j=t-\hat{\tau}}^{t-1}\eta_j^2}{2\mu_{\Phi}^2} + \frac{3(L+1+L^2 \hat{\tau})\sum_{t=\hat{\tau}+1}^T\eta_t^2}{2\mu_{\Phi}^2} + 3G^2\sum_{t=\hat{\tau}+1}^T\eta_t^2} \\ \nonumber
& + f(\x_{\hat{\tau}+1}) - f(\x_{T+1}) + \frac{\hat{\tau}G^2\sum_{t=\hat{\tau}+1}^T\eta_t^2}{2} + \lrincir{\frac{3G^2}{2}+\sigma_{\Vcal}^2} C_{\P} \sum_{t=\hat{\tau}+1}^T \eta_t \cdot \rho^{\hat{\tau}}.
\end{align} $\textcircled{1}$ is obtained due to Theorem \ref{theorem:regret:nonconvex:analysis}. Putting it back into \ref{equa:regret:marchon:proof:convex}, we thus completes the proof.
\end{proof}

\subsection{Proof of Theorem \ref{theorem:regret:strongly:convex:analysis}}

\textbf{Theorem} \ref{theorem:regret:strongly:convex:analysis}. Suppose $\{\x_t\}_{t=1}^T$ is yielded by Algorithm \ref{algo:marchon}, choose non-increasing positive sequence $\{\eta_t\}_{t=1}^T$. Let $\hat{\tau} \ge \tau$ for any $t\in[T]$. Under Assumptions \ref{assumption:markov:chain}-\ref{assumption:bounded:domain}, we have
\begin{align}
\nonumber
& \Rcal^{\textsc{MarchOn}}_{\zeta_T}(\x^\ast) \\ \nonumber
\le & C_5 \sum_{t=1}^T \eta_t + \frac{3G^2}{\mu_{\Phi}}\sum_{t=1}^{\hat{\tau}} \eta_t + \frac{3\lrincir{f(\x_{\hat{\tau}+1}) - f(\x_{T+1})}}{\mu_{\Phi}} + \sum_{t=\hat{\tau}+1}^T \lrincir{\frac{3C_1\hat{\tau}}{\mu_{\Phi}}  \sum_{j=t-\hat{\tau}}^{t-1}\eta_j^2 + \frac{3\lrincir{C_2+C_3\hat{\tau}}}{\mu_{\Phi}}\eta_t^2 + \frac{3C_4 C_{\P}\rho^{\hat{\tau}}}{\mu_{\Phi}} \eta_t} \\ \nonumber 
& + \frac{\mu_{\Phi} R^2}{2} \sum_{t=2}^T \lrincir{\frac{1}{\eta_t} - \frac{1}{\eta_{t-1}} - \frac{\mu_f}{\mu_{\Phi}}} + \frac{R^2}{\eta_1},
\end{align}  when $f_v$ is $\mu_f$-strongly convex for any $v\in\Vcal$.
\begin{proof}
Given any $i_t\in[n]$, we have 
\begin{align}
\nonumber
& f_{i_t}(\x_t) -f_{i_t}(\x^\ast) \\ \nonumber
= & \lrangle{\nabla f_{i_t}(\x_t), \x_t - \x^\ast } - \lrincir{ f_{i_t}(\x^\ast) - f_{i_t}(\x_t) - \lrangle{\nabla f_{i_t}(\x_t), \x^\ast - \x_t} } \\ \nonumber
= & \lrangle{\nabla f_{i_t}(\x_t), \x_t - \x^\ast } - B_{f_{i_t}}(\x^\ast, \x_t) \\ \nonumber 
= & \EE_{\a\sim\Dcal_{i_t}}\lrangle{\nabla f_{i_t}(\x_t;\a), \x_t - \x^\ast } - B_{f_{i_t}}(\x^\ast, \x_t) \\ \nonumber
= & \EE_{\a\sim\Dcal_{i_t}} \lrangle{\nabla f_{i_t}(\x_t;\a), \x_t - \x_{t+1}} + \EE_{\a\sim\Dcal_{i_t}} \lrangle{\nabla f_{i_t}(\x_t;\a), \x_{t+1} - \x^\ast }  - B_{f_{i_t}}(\x^\ast, \x_t)\\ \nonumber
\overset{\textcircled{1}}{\le} & \EE_{\a\sim\Dcal_{i_t}} \lrangle{\nabla f_{i_t}(\x_t;\a), \x_t - \x_{t+1}} + \EE_{\a\sim\Dcal_{i_t}} \frac{B_{\Phi}(\x^\ast, \x_t) -  B_{\Phi}(\x^\ast, \x_{t+1}) - B_{\Phi}(\x_{t+1}, \x_t)}{\eta_t}  - B_{f_{i_t}}(\x^\ast, \x_t)\\ \nonumber
\overset{\textcircled{2}}{\le} & \EE_{\a\sim\Dcal_{i_t}} \lrangle{\nabla f_{i_t}(\x_t;\a), \x_t - \x_{t+1}}  - \frac{\mu_{\Phi}}{2\eta_t} \EE_{\a\sim\Dcal_{i_t}} \lrnorm{\x_{t+1} - \x_t}^2 + \EE_{\a\sim\Dcal_{i_t}} \frac{B_{\Phi}(\x^\ast, \x_t) -  B_{\Phi}(\x^\ast, \x_{t+1})}{\eta_t}  - B_{f_{i_t}}(\x^\ast, \x_t) \\ \nonumber
\overset{\textcircled{3}}{\le} & \frac{\eta_t}{\mu_{\Phi}}\cdot\EE_{\a\sim\Dcal_{i_t}} \lrnorm{\nabla f_{i_t}(\x_t;\a)}^2 -\frac{\mu_{\Phi}}{4\eta_t}\cdot \EE_{\a\sim\Dcal_{i_t}} \lrnorm{\x_{t+1}-\x_t}^2 + \EE_{\a\sim\Dcal_{i_t}} \frac{B_{\Phi}(\x^\ast, \x_t) -  B_{\Phi}(\x^\ast, \x_{t+1})}{\eta_t}  - B_{f_{i_t}}(\x^\ast, \x_t)\\ \nonumber
\le & \frac{\eta_t}{\mu_{\Phi}}\cdot \EE_{\a\sim\Dcal_{i_t}} \lrnorm{\nabla f_{i_t}(\x_t;\a)}^2 + \EE_{\a\sim\Dcal_{i_t}} \frac{B_{\Phi}(\x^\ast, \x_t) -  B_{\Phi}(\x^\ast, \x_{t+1})}{\eta_t} - B_{f_{i_t}}(\x^\ast, \x_t).
\end{align} $\textcircled{1}$ holds due to Lemma \ref{lemma_mirror_descent_update_rule} by setting $\g = \nabla f_{i_t}(\x_t)$, $\u_t = \x_t$, $\u_{t+1} = \x_{t+1}$, $\u^\ast = \x^\ast$, and $\lambda = \eta_t$. $\textcircled{2}$ holds due to $\Phi$ is $\mu_{\Phi}$-strongly convex.  $\textcircled{3}$  holds because $\lrangle{\u,\v} \le \frac{a}{2} \lrnorm{\u}^2 + \frac{1}{2a}\lrnorm{\v}^2$ holds for any $\u$, $\v$, and $a>0$. 

Telescoping it over $t$, we have
\begin{align}
\nonumber
&\Rcal^{\textsc{MarchOn}}_{\zeta_T}(\x^\ast) = \EE \sum_{t=1}^T \lrincir{f_{i_t}(\x_t) -f_{i_t}(\x^\ast)}  \\ \nonumber
\le & \sum_{t=1}^T\frac{\eta_t}{\mu_{\Phi}}\EE_{\a\sim\Dcal_{i_t}} \lrnorm{\nabla f_{i_t}(\x_t;\a)}^2+ \sum_{t=1}^T \EE_{\a\sim\Dcal_{i_t}} \frac{ B_{\Phi}(\x^\ast, \x_t) -  B_{\Phi}(\x^\ast, \x_{t+1}) }{\eta_t}  - \sum_{t=1}^T B_{f_{i_t}}(\x^\ast, \x_t)\\ \nonumber 
\le & \sum_{t=1}^T\frac{\eta_t}{\mu_{\Phi}}\EE_{\a\sim\Dcal_{i_t}} \lrnorm{\nabla f_{i_t}(\x_t;\a)}^2 + \sum_{t=2}^T B_{\Phi}(\x^\ast, \x_t)\lrincir{\frac{1}{\eta_t} - \frac{1}{\eta_{t-1}}} + \frac{B_{\Phi}(\x^\ast, \x_1)}{\eta_1} - \sum_{t=1}^T B_{f_{i_t}}(\x^\ast, \x_t)\\ \nonumber 
\le & \sum_{t=1}^T\frac{\eta_t}{\mu_{\Phi}}\EE_{\a\sim\Dcal_{i_t}} \lrnorm{\nabla f_{i_t}(\x_t;\a)}^2 + \sum_{t=2}^T \frac{\mu_{\Phi}}{2}\lrincir{\frac{1}{\eta_t} - \frac{1}{\eta_{t-1}}}\cdot \lrnorm{\x_t - \x_{\ast}}^2 + \frac{R^2}{\eta_1} - \frac{\mu_f}{2} \lrnorm{\x_t - \x_{\ast}}^2\\ \label{equa:regret:marchon:proof} 
= & \sum_{t=1}^T\frac{\eta_t}{\mu_{\Phi}}\EE_{\a\sim\Dcal_{i_t}} \lrnorm{\nabla f_{i_t}(\x_t;\a)}^2 + \frac{\mu_{\Phi}}{2} \sum_{t=2}^T \lrincir{\frac{1}{\eta_t} - \frac{1}{\eta_{t-1}} - \frac{\mu_f}{\mu_{\Phi}}}\cdot \lrnorm{\x_t - \x_{\ast}}^2 + \frac{R^2}{\eta_1} \\ \label{equa:regret:marchon:proof:strongly:convex}
\le & \underbrace{\sum_{t=1}^T\frac{\eta_t}{\mu_{\Phi}}\EE_{\a\sim\Dcal_{i_t}} \lrnorm{\nabla f_{i_t}(\x_t;\a)}^2}_{I_1} + \frac{\mu_{\Phi} R^2}{2} \sum_{t=2}^T \lrincir{\frac{1}{\eta_t} - \frac{1}{\eta_{t-1}} - \frac{\mu_f}{\mu_{\Phi}}} + \frac{R^2}{\eta_1}.
\end{align} 

Focusing on $I_1$, we obtain
\begin{align}
\nonumber
& \EE I_1 = \EE \sum_{t=1}^T\frac{\eta_t}{\mu_{\Phi}}\lrnorm{\nabla f_{i_t}(\x_t;\a) - \nabla f_{i_t}(\x_t) + \nabla f_{i_t}(\x_t) - \nabla f(\x_t) + \nabla f(\x_t)}^2 \\ \nonumber
\le & \EE \sum_{t=1}^T\frac{3\eta_t \lrnorm{\nabla f_{i_t}(\x_t;\a) - \nabla f_{i_t}(\x_t)}^2}{\mu_{\Phi}} + \EE \sum_{t=1}^T\frac{3\eta_t \lrnorm{\nabla f_{i_t}(\x_t) - \nabla f(\x_t)}^2}{\mu_{\Phi}} + \EE \sum_{t=1}^T\frac{3\eta_t \lrnorm{\nabla f(\x_t)}^2}{\mu_{\Phi}}\\ \nonumber
\le & \frac{3\sigma_v^2}{\mu_{\Phi}} \sum_{t=1}^T \eta_t + \frac{3\sigma_{\Vcal}^2}{\mu_{\Phi}} \sum_{t=1}^T \eta_t + \EE \sum_{t=1}^T\frac{3\eta_t \lrnorm{\nabla f(\x_t)}^2}{\mu_{\Phi}} \\ \nonumber
\overset{\textcircled{1}}{\le} & \frac{3\sigma_v^2+3\sigma_{\Vcal}^2}{\mu_{\Phi}} \sum_{t=1}^T \eta_t +  LG \sum_{t=\hat{\tau}+1}^T \eta_t^2 + \frac{3LG \lrincir{\sigma_v^2 + \sigma_{\Vcal}^2 + G^2} \hat{\tau}}{\mu_{\Phi}^2} \sum_{t=\hat{\tau}+1}^T \sum_{j=t-\hat{\tau}}^{t-1}\eta_j^2 \\ \nonumber
& + \lrincir{\sigma_v^2 + \sigma_{\Vcal}^2+G^2} \lrincir{ \frac{L^2\hat{\tau}\sum_{t=\hat{\tau}+1}^T \sum_{j=t-\hat{\tau}}^{t-1}\eta_j^2}{2\mu_{\Phi}^2} + \frac{3(L+1+L^2 \hat{\tau})\sum_{t=\hat{\tau}+1}^T\eta_t^2}{2\mu_{\Phi}^2} + 3G^2\sum_{t=\hat{\tau}+1}^T\eta_t^2} \\ \nonumber
& + f(\x_{\hat{\tau}+1}) - f(\x_{T+1}) + \frac{\hat{\tau}G^2\sum_{t=\hat{\tau}+1}^T\eta_t^2}{2} + \lrincir{\frac{3G^2}{2}+\sigma_{\Vcal}^2} C_{\P} \sum_{t=\hat{\tau}+1}^T \eta_t \cdot \rho^{\hat{\tau}}.
\end{align} $\textcircled{1}$ is obtained due to Theorem \ref{theorem:regret:nonconvex:analysis}. Putting it back into \ref{equa:regret:marchon:proof:strongly:convex}, we finally complete the proof.
\end{proof}

\subsection{Proof of Theorem \ref{theorem:regret:nonconvex:analysis}}
\textbf{Theorem} \ref{theorem:regret:nonconvex:analysis}. Suppose $\{\x_t\}_{t=1}^T$ is yielded by Algorithm \ref{algo:marchon}, choose $\eta_t = \eta>0$ for any $t\in[T]$, and let $\hat{\tau} \ge \tau$ for any $t\in[T]$. Under Assumptions \ref{assumption:markov:chain}-\ref{assumption:bounded:domain}, we have
\begin{align}
\nonumber
\EE \sum_{t=1}^T \lrnorm{\nabla f(\x_t)}^2  \le G^2\hat{\tau} + \frac{f(\x_{\hat{\tau}+1}) - f(\x_{T+1})}{\eta} + \lrincir{C_1 \hat{\tau}^2 \eta + \lrincir{C_2+C_3\hat{\tau}}\eta + C_4 C_{\P} \rho^{\hat{\tau}}}(T - \hat{\tau}).
\end{align} 
\begin{proof}
Setting $\eta_t = \eta$ in Lemma \ref{lemma:regret:nonconvex:analysis}, and then rearranging items, we thus complete the proof.
\end{proof}

\subsection{Useful lemmas}
\begin{Lemma} 
\label{lemma:regret:nonconvex:analysis}
Suppose $\{\x_t\}_{t=1}^T$ is yielded by Algorithm \ref{algo:marchon}, choose non-increasing positive sequence $\{\eta_t\}_{t=1}^T$, and let $\hat{\tau} \ge \tau$ for any $t\in[T]$. Under Assumptions \ref{assumption:markov:chain}-\ref{assumption:bounded:domain}, we have
\begin{align}
\nonumber
& \EE \sum_{t=1}^T \eta_t \lrnorm{\nabla f(\x_t)}^2  \\ \nonumber
\le & G^2\sum_{t=1}^{\hat{\tau}} \eta_t + f(\x_{\hat{\tau}+1}) - f(\x_{T+1}) + \sum_{t=\hat{\tau}+1}^T \lrincir{C_1 \hat{\tau} \sum_{j=t-\hat{\tau}}^{t-1}\eta_j^2 + \lrincir{C_2+C_3\hat{\tau}}\eta_t^2 + C_4 C_{\P} \eta_t \rho^{\hat{\tau}}}.
\end{align} 
\end{Lemma}
\begin{proof}
\begin{align}
\nonumber
& \sum_{t=1}^T \eta_t \lrnorm{\nabla f(\x_t)}^2 \\ \nonumber
= & \sum_{t=1}^{\hat{\tau}} \eta_t \lrnorm{\nabla f(\x_t)}^2 + \sum_{t=\hat{\tau}+1}^T \eta_t \lrincir{\lrnorm{\nabla f(\x_t)}^2 - \lrnorm{\nabla f(\x_{t-\hat{\tau}})}^2} + \sum_{t=\hat{\tau}+1}^T \eta_t \lrnorm{\nabla f(\x_{t-\hat{\tau}})}^2 \\ \label{equa:gradf:xt:norm}
\le & G^2\sum_{t=1}^{\hat{\tau}} \eta_t + \underbrace{\sum_{t=\hat{\tau}+1}^T \eta_t \lrincir{\lrnorm{\nabla f(\x_t)}^2 - \lrnorm{\nabla f(\x_{t-\hat{\tau}})}^2}}_{I_2} + \underbrace{\sum_{t=\hat{\tau}+1}^T \eta_t \lrnorm{\nabla f(\x_{t-\hat{\tau}})}^2}_{I_3}.
\end{align} Consider to bound $I_2$, and we have 
\begin{align}
\nonumber
& I_2 = \sum_{t=\hat{\tau}+1}^T \eta_t \lrincir{\lrnorm{\nabla f(\x_t)} + \lrnorm{\nabla f(\x_{t-\hat{\tau}})}} \cdot \lrincir{\lrnorm{\nabla f(\x_t)} - \lrnorm{\nabla f(\x_{t-\hat{\tau}})}} \\ \nonumber
\le & 2G \cdot \sum_{t=\hat{\tau}+1}^T \eta_t  \cdot \lrincir{\lrnorm{\nabla f(\x_t)} - \lrnorm{\nabla f(\x_{t-\hat{\tau}})}} \\ \nonumber
\le & 2G \cdot \sum_{t=\hat{\tau}+1}^T \eta_t  \cdot \lrnorm{\nabla f(\x_t) - \nabla f(\x_{t-\hat{\tau}})} \\ \nonumber
\le & 2LG \cdot \sum_{t=\hat{\tau}+1}^T \eta_t  \cdot \lrnorm{\x_t - \x_{t-\hat{\tau}}} \\ \label{equa:I2}
\le & LG \cdot \sum_{t=\hat{\tau}+1}^T \lrincir{\eta_t^2 + \lrnorm{\x_t - \x_{t-\hat{\tau}}}^2}.
\end{align}  According to Lemma \ref{lemma:diff:x:tau}, we have 
\begin{align}
\nonumber
\EE \lrnorm{\x_t - \x_{t-\hat{\tau}}}^2 \le \frac{3\hat{\tau} \lrincir{\sigma_v^2 + \sigma_{\Vcal}^2 + G^2}}{\mu_{\Phi}^2} \sum_{j=t-\hat{\tau}}^{t-1}\eta_j^2.
\end{align} Putting it back into \ref{equa:I2}, we obtain
\begin{align}
\label{equa:I2:expectation}
\EE I_2 \le LG \sum_{t=\hat{\tau}+1}^T \eta_t^2 + \frac{3LG \lrincir{\sigma_v^2 + \sigma_{\Vcal}^2 + G^2} \hat{\tau}}{\mu_{\Phi}^2} \sum_{t=\hat{\tau}+1}^T \sum_{j=t-\hat{\tau}}^{t-1}\eta_j^2. 
\end{align}

Now, we begin to bound $I_3$. When $\hat{\tau}$ is sufficiently large such that $\hat{\tau} \ge \tau$, we obtain
\begin{align}
\nonumber
\lrnorm{\bPi_\ast - \P^{\hat{\tau}}}_{\infty} \le C_{\P} \cdot \rho^{\hat{\tau}}
\end{align} due to Lemma \ref{lemma:distance:stationary:probabilition}. Furthermore, according to Lemma \ref{lemma:grad:fxt:taut}, $I_3$ is bounded by
\begin{align}
\nonumber
\EE I_3 \le & \lrincir{\sigma_v^2 + \sigma_{\Vcal}^2+G^2} \lrincir{ \frac{L^2\hat{\tau}\sum_{t=\hat{\tau}+1}^T \sum_{j=t-\hat{\tau}}^{t-1}\eta_j^2}{2\mu_{\Phi}^2} + \frac{3(L+1+L^2 \hat{\tau})\sum_{t=\hat{\tau}+1}^T\eta_t^2}{2\mu_{\Phi}^2} + 3G^2\sum_{t=\hat{\tau}+1}^T\eta_t^2} \\ \label{equa:I3:expectation}
& + f(\x_{\hat{\tau}+1}) - f(\x_{T+1}) + \frac{\hat{\tau}G^2\sum_{t=\hat{\tau}+1}^T\eta_t^2}{2} + \lrincir{\frac{3G^2}{2}+\sigma_{\Vcal}^2} C_{\P} \sum_{t=\hat{\tau}+1}^T \eta_t \cdot \rho^{\hat{\tau}}.
\end{align}
Putting \ref{equa:I2:expectation} and \ref{equa:I3:expectation} back into \ref{equa:gradf:xt:norm}, we obtain 
\begin{align}
\nonumber
& \EE \sum_{t=1}^T \eta_t \lrnorm{\nabla f(\x_t)}^2 \\ \nonumber 
\le & LG \sum_{t=\hat{\tau}+1}^T \eta_t^2 + \frac{3LG \lrincir{\sigma_v^2 + \sigma_{\Vcal}^2 + G^2} \hat{\tau}}{\mu_{\Phi}^2} \sum_{t=\hat{\tau}+1}^T \sum_{j=t-\hat{\tau}}^{t-1}\eta_j^2 \\ \nonumber
& + \lrincir{\sigma_v^2 + \sigma_{\Vcal}^2+G^2} \lrincir{ \frac{L^2\hat{\tau}\sum_{t=\hat{\tau}+1}^T \sum_{j=t-\hat{\tau}}^{t-1}\eta_j^2}{2\mu_{\Phi}^2} + \frac{3(L+1+L^2 \hat{\tau})\sum_{t=\hat{\tau}+1}^T\eta_t^2}{2\mu_{\Phi}^2} + 3G^2\sum_{t=\hat{\tau}+1}^T\eta_t^2} \\ \nonumber
& + f(\x_{\hat{\tau}+1}) - f(\x_{T+1}) + \frac{\hat{\tau}G^2\sum_{t=\hat{\tau}+1}^T\eta_t^2}{2} + \lrincir{\frac{3G^2}{2}+\sigma_{\Vcal}^2} C_{\P} \sum_{t=\hat{\tau}+1}^T \eta_t \cdot \rho^{\hat{\tau}}.
\end{align} It thus completes the proof.
\end{proof}

\begin{Lemma}[Lemma 1 in \cite{zhao2020understand}]
\label{lemma_mirror_descent_update_rule}
\nonumber
Given any vectors $\g$, $\u_t\in\Xcal$, $\u^\ast\in\Xcal$ , and a constant scalar $\lambda>0$, if 
\begin{align}
\nonumber
\u_{t+1} = \argmin_{\u\in\Xcal} \lrangle{\g, \u - \u_t} + \frac{1}{\lambda} B_{\Phi}(\u, \u_t),
\end{align} we have
\begin{align}
\nonumber
\lrangle{\g, \u_{t+1} - \u^\ast} \le \frac{1}{\lambda}\lrincir{ B_{\Phi}(\u^\ast, \u_t) -  B_{\Phi}(\u^\ast, \u_{t+1}) - B_{\Phi}(\u_{t+1}, \u_t) }.
\end{align}
\end{Lemma}
\begin{proof}

Denote $h(\u) = \lrangle{\g, \u-\u_t} + \frac{1}{\lambda}B_{\Phi}(\u, \u_t)$, and $\u_{\tau} = \u_{t+1} + \tau (\u^\ast - \u_{t+1})$. According to the optimality of $\x_t$, we have
\begin{align}
\nonumber
&0 \le  h(\u_\tau) - h(\u_{t+1}) \\ \nonumber
= & \lrangle{\g, \u_\tau - \u_{t+1}} + \frac{1}{\lambda}\lrincir{B_{\Phi}(\u_\tau, \u_t) - B_{\Phi}(\u_{t+1}, \u_t)} \\ \nonumber
= & \lrangle{\g, \tau (\u^\ast - \u_{t+1})} + \frac{1}{\lambda}\lrincir{ \Phi(\u_\tau) - \Phi(\u_{t+1})  + \lrangle{\nabla \Phi(\u_t), \tau (\u_{t+1} - \u^\ast)} } \\ \nonumber
\le & \lrangle{\g, \tau (\u^\ast - \u_{t+1})} + \frac{1}{\lambda} \lrangle{\nabla \Phi(\u_{t+1}), \tau (\u^\ast - \u_{t+1})} + \frac{1}{\lambda} \lrangle{\nabla \Phi(\u_t), \tau (\u_{t+1} - \u^\ast)}  \\ \nonumber
= & \lrangle{\g, \tau (\u^\ast - \u_{t+1})} + \frac{1}{\lambda} \lrangle{\nabla \Phi(\u_t)-\Phi(\u_{t+1}), \tau (\u_{t+1} - \u^\ast)}.
\end{align} Thus, we have
\begin{align}
\nonumber
& \lrangle{\g, \u_{t+1} - \u^\ast} \le \frac{1}{\lambda} \lrangle{\nabla \Phi(\u_t)-\Phi(\u_{t+1}), \u_{t+1} - \u^\ast}  \\ \nonumber
= & \frac{1}{\lambda}\lrincir{ B_{\Phi}(\u^\ast, \u_t) -  B_{\Phi}(\u^\ast, \u_{t+1}) - B_{\Phi}(\u_{t+1}, \u_t) }.
\end{align} It completes the proof.
\end{proof}

\begin{Lemma}
\label{lemma:diff:x:tau}
Recall the update rule, that is \ref{equa:parameters:update}, 
\begin{align}
\nonumber
\x_{t+1} = \argmin_{\x} \lrangle{\nabla f_{i_t}(\x_t;\a\sim\Dcal_{i_t}), \x-\x_t} + \frac{1}{\eta_t}B_{\Phi}(\x, \x_t).
\end{align} Then, we have 
\begin{align}
\nonumber
\EE \lrnorm{\x_{t+\tau} - \x_t}^2 \le  \frac{3\tau \lrincir{\sigma_v^2 + \sigma_{\Vcal}^2 + G^2}}{\mu_{\Phi}^2} \cdot \sum_{j=t}^{t+\tau-1}\eta_j^2
\end{align} for any $t\in[T]$ and $\tau\in\N_+$.
\end{Lemma}
\begin{proof}
\begin{align}
\nonumber
& \EE \lrnorm{\x_{t+\tau} - \x_t}^2 \\ \nonumber
= & \EE \lrnorm{\sum_{j=t}^{t+\tau-1} \lrincir{\x_{j+1} - \x_j}}^2 \\ \nonumber
= & \tau^2 \cdot \EE \lrnorm{\frac{1}{\tau}\sum_{j=t}^{t+\tau-1} \lrincir{\x_{j+1} - \x_j}}^2 \\ \nonumber
\le & \tau \cdot \sum_{j=t}^{t+\tau-1} \EE \lrnorm{\lrincir{\x_{j+1} - \x_j}}^2 \\ \nonumber
\overset{\textcircled{1}}{\le} & \tau \cdot \sum_{j=t}^{t+\tau-1} \frac{\eta_j^2 \cdot \EE  \lrnorm{\nabla f_{i_j}(\x_j;\a)}^2}{\mu_{\Phi}^2} \\ \nonumber
= & \tau \cdot \sum_{j=t}^{t+\tau-1} \frac{\eta_j^2 \EE \lrnorm{\nabla f_{i_j}(\x_j;\a) - \nabla f_{i_j}(\x_j) + \nabla f_{i_j}(\x_j) - \nabla f(\x_j) + \nabla f(\x_j)}^2}{\mu_{\Phi}^2} \\ \nonumber
\le & \tau \cdot \sum_{j=t}^{t+\tau-1} \frac{\eta_j^2 \lrincir{3\EE \lrnorm{\nabla f_{i_j}(\x_j;\a) - \nabla f_{i_j}(\x_j)}^2 + 3\EE\lrnorm{\nabla f_{i_j}(\x_j) - \nabla f(\x_j)}^2 + 3\EE\lrnorm{\nabla f(\x_j)}^2}}{\mu_{\Phi}^2} \\ \nonumber
\le & 3\tau \cdot \sum_{j=t}^{t+\tau-1} \frac{\eta_j^2 \lrincir{\sigma_v^2 + \sigma_{\Vcal}^2 + G^2}}{\mu_{\Phi}^2}.
\end{align} $\textcircled{1}$ holds due to Lemma \ref{lemma:diff:norm:xk}.  It thus finishes the proof.
\end{proof}

\begin{Lemma}
\label{lemma:diff:norm:xk}
For any $t\in[T]$, and $\x_t\in\Xcal_t$, when
\begin{align}
\nonumber
\x_{t+1} = \argmin_{\x\in\Xcal_t} \lrangle{\nabla f_{i_t}(\x_t;\a\sim\Dcal_{i_t}), \x - \x_t} + \frac{1}{\eta_t}B_{\Phi}(\x, \x_t), 
\end{align} we obtain
\begin{align}
\nonumber
\lrnorm{\x_{t+1} - \x_t} \le \frac{\eta_t \lrnorm{\nabla f_{i_t}(\x_t;\a)}}{\mu_{\Phi}}.
\end{align}
\end{Lemma}
\begin{proof}
According to Lemma \ref{lemma_mirror_descent_update_rule}, we have 
\begin{align}
\nonumber
\lrangle{\nabla f_{i_t}(\x_t;\a), \x_{t+1}-\x_t} \le & \frac{1}{\eta_t}\lrincir{B_{\Phi}(\x_t, \x_t) - B_{\Phi}(\x_t, \x_{t+1}) - B_{\Phi}(\x_{t+1}, \x_t)} \\ \nonumber 
= & \frac{1}{\eta_t}\lrincir{- B_{\Phi}(\x_t, \x_{t+1}) - B_{\Phi}(\x_{t+1}, \x_t)}
\end{align} by letting $\g = \nabla f_{i_t}(\x_t;\a)$, $\lambda = \eta_t$, $\u^\ast = \x_t$, $\u_{t+1} = \x_{t+1}$, and $\u_t = \x_t$ in Lemma \ref{lemma_mirror_descent_update_rule}. Therefore, we obtain
\begin{align}
\nonumber
B_{\Phi}(\x_t, \x_{t+1})  + B_{\Phi}(\x_{t+1}, \x_t) \le \lrangle{\eta_t \nabla f_{i_t}(\x_t;\a), \x_t - \x_{t+1}} \le \eta_t \lrnorm{\nabla f_{i_t}(\x_t;\a)} \cdot \lrnorm{\x_{t+1} - \x_t}.
\end{align} After that, we have
\begin{align}
\nonumber
&\lrnorm{\x_{t+1} - \x_t}^2 \\ \nonumber 
= & \frac{1}{2}\lrincir{\lrnorm{\x_{t+1} - \x_t}^2 + \lrnorm{\x_t - \x_{t+1}}^2} \\ \nonumber
\overset{\textcircled{1}}{\le} & \frac{1}{\mu_{\Phi}} \lrincir{B_{\Phi}(\x_{t+1}, \x_t) + B_{\Phi}(\x_t, \x_{t+1})} \\ \nonumber 
\le & \frac{\eta_t \lrnorm{\nabla f_{i_t}(\x_t;\a)} \cdot \lrnorm{\x_{t+1} - \x_t}}{\mu_{\Phi}}. 
\end{align} $\textcircled{1}$ holds because $\Phi$ is $\mu_{\Phi}$ strongly convex. Thus, we obtain
\begin{align}
\nonumber
\lrnorm{\x_{t+1} - \x_t} \le \frac{\eta_t \lrnorm{\nabla f_{i_t}(\x_t;\a)}}{\mu_{\Phi}}. 
\end{align} It finally completes the proof.
\end{proof} 

\begin{Lemma}[Lemma 1 in \cite{sun:2018:mcgd}]
\label{lemma:distance:stationary:probabilition}
Under Assumption \ref{assumption:markov:chain}, we obtain
\begin{align}
\nonumber
\lrnorm{\P^t - \bPi_\ast}_\infty \le C_{\P} \cdot \rho^t
\end{align} for $t \ge \tau$, where $C_{\P}$ is denoted in \ref{equa:define:CP}, and $\tau$ is denoted in \ref{equa:define:tau}. 
\end{Lemma}

\begin{Lemma}
\label{lemma:grad:fxt:taut}
Given a sufficient large $\hat{\tau} \ge 0$, when 
\begin{align}
\nonumber
\lrnorm{\bPi_\ast - \P^{\hat{\tau}}}_{\infty} \le \alpha(t)
\end{align} holds, we obtain
\begin{align}
\nonumber
& \EE \sum_{t=\hat{\tau}+1}^T \eta_t \lrnorm{\nabla f(\x_{t-\hat{\tau}})}^2 \\ \nonumber 
\le & \lrincir{\sigma_v^2 + \sigma_{\Vcal}^2+G^2} \lrincir{ \frac{L^2\hat{\tau}\sum_{t=\hat{\tau}+1}^T \sum_{j=t-\hat{\tau}}^{t-1}\eta_j^2}{2\mu_{\Phi}^2} + \frac{3(L+1+L^2 \hat{\tau})\sum_{t=\hat{\tau}+1}^T\eta_t^2}{2\mu_{\Phi}^2} + 3G^2\sum_{t=\hat{\tau}+1}^T\eta_t^2} \\ \nonumber
& + f(\x_{\hat{\tau}+1}) - f(\x_{T+1}) + \frac{\hat{\tau}G^2\sum_{t=\hat{\tau}+1}^T\eta_t^2}{2} + \lrincir{\frac{3G^2}{2}+\sigma_{\Vcal}^2} \sum_{t=\hat{\tau}+1}^T \eta_t \cdot \alpha(t).
\end{align} 
\end{Lemma}
\begin{proof}
Denote $\chi_t := \left \{\x_1, \x_2, \cdots, \x_t, i_1, \cdots, i_t \right\}$.
\begin{align}
\nonumber
& \EE_{i_t} \lrincir{\lrangle{\nabla f(\x_{t-\hat{\tau}}), \nabla f_{i_t}(\x_{t-\hat{\tau}})} | \chi_{t-\hat{\tau}}} \\ \nonumber 
= & \sum_{i=1}^n\lrincir{\lrangle{\nabla f(\x_{t-\hat{\tau}}), \nabla f_{i_t}(\x_{t-\hat{\tau}})}} \cdot \PP\lrincir{i_t=i | i_{t-\hat{\tau}}}\\ \nonumber
= & \sum_{i=1}^n \lrincir{\lrangle{\nabla f(\x_{t-\hat{\tau}}), \nabla f_i(\x_{t-\hat{\tau}})}} \cdot  [\P^{\hat{\tau}}]_{i_{t-\hat{\tau}}, i}\\ \nonumber
= & \lrnorm{\nabla f(\x_{t-\hat{\tau}})}^2 +  \lrangle{\nabla f(\x_{t-\hat{\tau}}), \lrincir{\sum_{i=1}^n \nabla f_i(\x_{t-\hat{\tau}}) \cdot  [\P^{\hat{\tau}}]_{i_{t-\hat{\tau}}, i}} - \nabla f(\x_{t-\hat{\tau}})}\\ \nonumber
= & \lrnorm{\nabla f(\x_{t-\hat{\tau}})}^2 +  \lrangle{\nabla f(\x_{t-\hat{\tau}}), \sum_{i=1}^n \lrincir{\nabla f_i(\x_{t-\hat{\tau}}) \cdot  [\P^{\hat{\tau}}]_{i_{t-\hat{\tau}}, i} - \frac{1}{n}\nabla f_i(\x_{t-\hat{\tau}})}}\\ \nonumber
= & \lrnorm{\nabla f(\x_{t-\hat{\tau}})}^2 +  \lrangle{\nabla f(\x_{t-\hat{\tau}}), \sum_{i=1}^n \nabla f_i(\x_{t-\hat{\tau}}) \cdot \lrincir{[\P^{\hat{\tau}}]_{i_{t-\hat{\tau}}, i} - \frac{1}{n}}}.
\end{align}   Thus, we obtain
\begin{align}
\label{equa:gradf:xttaut:norm}
& \EE \sum_{t=\hat{\tau}+1}^T \eta_t \lrnorm{\nabla f(\x_{t-\hat{\tau}})}^2 \\ \nonumber 
= & \underbrace{\EE \lrincir{\sum_{t=\hat{\tau}+1}^T \eta_t \lrangle{\nabla f(\x_{t-\hat{\tau}}), \nabla f_{i_t}(\x_{t-\hat{\tau}}) } }}_{J_1} + \underbrace{\EE \sum_{t=\hat{\tau}+1}^T \eta_t \lrangle{\nabla f(\x_{t-\hat{\tau}}), \sum_{i=1}^n \nabla f_i(\x_{t-\hat{\tau}}) \cdot \lrincir{\frac{1}{n} - [\P^{\hat{\tau}}]_{i_{t-\hat{\tau}}, i}}}}_{J_2}.
\end{align} Now, we start to bound $J_1$. 
\begin{align}
\nonumber
& f(\x_{t+1})\le f(\x_t) + \lrangle{\nabla f(\x_t), \x_{t+1} - \x_t} + \frac{L}{2}\lrnorm{\x_{t+1} - \x_t}^2 \\ \nonumber
\le & f(\x_t) + \lrangle{\nabla f(\x_{t-\hat{\tau}}), \x_{t+1} - \x_t} + \lrangle{\nabla f(\x_t) - \nabla f(\x_{t-\hat{\tau}}), \x_{t+1} - \x_t} + \frac{L}{2}\lrnorm{\x_{t+1} - \x_t}^2 \\ \nonumber
\le & f(\x_t) + \lrangle{\nabla f(\x_{t-\hat{\tau}}), \x_{t+1} - \x_t} + \frac{\lrnorm{\nabla f(\x_t) - \nabla f(\x_{t-\hat{\tau}})}^2}{2} + \frac{\lrnorm{\x_{t+1} - \x_t}^2}{2} + \frac{L}{2}\lrnorm{\x_{t+1} - \x_t}^2 \\ \nonumber
\le & f(\x_t) + \lrangle{\nabla f(\x_{t-\hat{\tau}}), \x_{t+1} - \x_t} + \frac{L^2\lrnorm{\x_t - \x_{t-\hat{\tau}}}^2}{2} + \frac{L+1}{2}\lrnorm{\x_{t+1} - \x_t}^2.
\end{align} Thus, we obtain
\begin{align}
\nonumber
\lrangle{\nabla f(\x_{t-\hat{\tau}}), \x_t - \x_{t+1}} \le f(\x_t) - f(\x_{t+1}) + \frac{L^2\lrnorm{\x_t - \x_{t-\hat{\tau}}}^2}{2} + \frac{L+1}{2}\lrnorm{\x_{t+1} - \x_t}^2.
\end{align}  Additionally, we obtain
\begin{align}
\nonumber
& \EE \lrangle{\nabla f(\x_{t-\hat{\tau}}), \x_t - \x_{t+1}} \\ \nonumber
= & \EE \lrangle{\nabla f(\x_{t-\hat{\tau}}), \eta_t\nabla f_{i_t}(\x_t; \a)} + \EE \lrangle{\nabla f(\x_{t-\hat{\tau}}), \x_t - \x_{t+1} - \eta_t\nabla f_{i_t}(\x_t; \a)} \\ \nonumber 
= & \underbrace{\EE \eta_t\lrangle{\nabla f(\x_{t-\hat{\tau}}), \nabla f_{i_t}(\x_t; \a) - \nabla f_{i_t}(\x_{t-\hat{\tau}}; \a)}}_{K_1} + \EE \eta_t\lrangle{\nabla f(\x_{t-\hat{\tau}}), \nabla f_{i_t}(\x_{t-\hat{\tau}}; \a)} \\ \label{equa:nablaf:xt:xt1}
& + \underbrace{\EE \lrangle{\nabla f(\x_{t-\hat{\tau}}), \x_t - \x_{t+1} - \eta_t\nabla f_{i_t}(\x_t; \a)}}_{K_2}.
\end{align} Now, we start to bound $K_1$. 
\begin{align}
\nonumber
& K_1 = \EE \eta_t\lrangle{\nabla f(\x_{t-\hat{\tau}}), \EE_{\a\sim\Dcal_{i_t}}\nabla f_{i_t}(\x_t; \a) - \EE_{\a\sim\Dcal_{i_t}}\nabla f_{i_t}(\x_{t-\hat{\tau}}; \a)} \\ \nonumber
= & \EE \eta_t\lrangle{\nabla f(\x_{t-\hat{\tau}}), \nabla f_{i_t}(\x_t) - \nabla f_{i_t}(\x_{t-\hat{\tau}})} \\ \nonumber
= & \EE \eta_t\lrangle{\nabla f(\x_{t-\hat{\tau}}), \sum_{j=t-\hat{\tau}}^{t-1}\lrincir{\nabla f_{i_t}(\x_{j+1}) - \nabla f_{i_t}(\x_j)}} \\ \nonumber
= & \EE \eta_t\sum_{j=t-\hat{\tau}}^{t-1}\lrangle{\nabla f(\x_{t-\hat{\tau}}), \nabla f_{i_t}(\x_{j+1}) - \nabla f_{i_t}(\x_j)} \\ \nonumber
\ge & - \EE \eta_t\sum_{j=t-\hat{\tau}}^{t-1}\lrnorm{\nabla f(\x_{t-\hat{\tau}})} \lrnorm{\nabla f_{i_t}(\x_{j+1}) - \nabla f_{i_t}(\x_j)} \\ \nonumber
\ge & - \EE \frac{\eta_t}{2}\sum_{j=t-\hat{\tau}}^{t-1}\lrincir{\eta_t\lrnorm{\nabla f(\x_{t-\hat{\tau}})}^2+ \frac{1}{\eta_t}\lrnorm{\nabla f_{i_t}(\x_{j+1}) - \nabla f_{i_t}(\x_j)}^2} \\ \nonumber
\ge & -\frac{\eta_t^2 \hat{\tau}G^2}{2}  - \frac{L^2}{2}\sum_{j=t-\hat{\tau}}^{t-1}\EE \lrnorm{\x_{j+1} - \x_j}^2 \\ \nonumber
\ge & -\frac{\eta_t^2 \hat{\tau}G^2}{2}  - \frac{\eta_t^2 L^2}{2\mu_{\Phi}^2}\sum_{j=t-\hat{\tau}}^{t-1}\EE \lrnorm{\nabla f_{i_j}(\x_j;\a\sim\Dcal_{i_j})}^2 \\ \nonumber
= & -\frac{\eta_t^2 \hat{\tau}G^2}{2}  - \frac{\eta_t^2 L^2}{2\mu_{\Phi}^2}\sum_{j=t-\hat{\tau}}^{t-1}\EE \lrnorm{\nabla f_{i_j}(\x_j;\a\sim\Dcal_{i_j}) - \nabla f_{i_j}(\x_j) + \nabla f_{i_j}(\x_j) - \nabla f(\x_j) + \nabla f(\x_j)}^2 \\ \label{equa:J3}
\ge & -\frac{\eta_t^2 \hat{\tau}G^2}{2}  - \frac{3\eta_t^2 L^2 \hat{\tau}\lrincir{\sigma_v^2 + \sigma_{\Vcal}^2 + G^2}}{2\mu_{\Phi}^2}. 
\end{align}

Now, we start to bound $K_2$. 
\begin{align}
\nonumber
K_2 \ge & - \EE \lrnorm{\nabla f(\x_{t-\hat{\tau}})} \lrnorm{\x_t - \x_{t+1} - \eta_t\nabla f_{i_t}(\x_t; \a)} \\ \nonumber
\ge & -G\eta_t^2\cdot\EE \lrnorm{\nabla f_{i_t}(\x_t; \a)}^2 \\ \nonumber
\ge & -G\eta_t^2\cdot\EE \lrnorm{\nabla f_{i_t}(\x_t; \a) - \nabla f_{i_t}(\x_t) + \nabla f_{i_t}(\x_t) - \nabla f(\x_t) + \nabla f(\x_t)}^2 \\ \label{equa:J4}
\ge & -3G\eta_t^2\lrincir{\sigma_v^2 + \sigma_{\Vcal}^2+G^2}.
\end{align} 

Putting \ref{equa:J3} and \ref{equa:J4} into \ref{equa:nablaf:xt:xt1}, we obtain
\begin{align}
\nonumber
& \EE \eta_t\lrangle{\nabla f(\x_{t-\hat{\tau}}), \nabla f_{i_t}(\x_{t-\hat{\tau}})} \\ \nonumber
= & \EE \eta_t\lrangle{\nabla f(\x_{t-\hat{\tau}}), \EE_{\a\sim\Dcal_{i_t}} \nabla f_{i_t}(\x_{t-\hat{\tau}}; \a)} \\ \nonumber
\le & \EE \lrangle{\nabla f(\x_{t-\hat{\tau}}), \x_t-\x_{t+1}} + \frac{\eta_t^2 \hat{\tau}G^2}{2}  + \frac{3\eta_t^2 L^2 \hat{\tau}\lrincir{\sigma_v^2 + \sigma_{\Vcal}^2 + G^2}}{2\mu_{\Phi}^2} + 3G\eta_t^2\lrincir{\sigma_v^2 + \sigma_{\Vcal}^2+G^2}\\ \nonumber
\le & f(\x_t) - f(\x_{t+1}) + \frac{L^2\EE\lrnorm{\x_t - \x_{t-\hat{\tau}}}^2}{2} + \frac{L+1}{2}\EE\lrnorm{\x_{t+1} - \x_t}^2 + \frac{\eta_t^2 \hat{\tau}G^2}{2} \\ \nonumber
& + \frac{3\eta_t^2 L^2 \hat{\tau}\lrincir{\sigma_v^2 + \sigma_{\Vcal}^2 + G^2}}{2\mu_{\Phi}^2} + 3G\eta_t^2\lrincir{\sigma_v^2 + \sigma_{\Vcal}^2+G^2} \\ \nonumber
\le & f(\x_t) - f(\x_{t+1}) + \frac{L^2\hat{\tau}\lrincir{\sigma_v^2 + \sigma_{\Vcal}^2+G^2}\sum_{j=t-\hat{\tau}}^{t-1}\eta_j^2}{2\mu_{\Phi}^2} + \frac{3(L+1)\lrincir{\sigma_v^2 + \sigma_{\Vcal}^2+G^2}\eta_t^2}{2\mu_{\Phi}^2} \\ \nonumber
& + \frac{\eta_t^2 \hat{\tau}G^2}{2}  + \frac{3\eta_t^2 L^2 \hat{\tau}\lrincir{\sigma_v^2 + \sigma_{\Vcal}^2 + G^2}}{2\mu_{\Phi}^2} + 3G\eta_t^2\lrincir{\sigma_v^2 + \sigma_{\Vcal}^2+G^2} \\ \nonumber
\le & f(\x_t) - f(\x_{t+1}) + \lrincir{\sigma_v^2 + \sigma_{\Vcal}^2+G^2}\lrincir{ \frac{L^2\hat{\tau}\sum_{j=t-\hat{\tau}}^{t-1}\eta_j^2}{2\mu_{\Phi}^2} + \frac{3(L+1+L^2\hat{\tau})\eta_t^2}{2\mu_{\Phi}^2} +3G^2\eta_t^2} + \frac{\eta_t^2\hat{\tau}G^2}{2}.
\end{align} Therefore, $J_1$ is bounded by 
\begin{align}
\nonumber
& J_1 = \EE \lrincir{\sum_{t=\hat{\tau}+1}^T \eta_t \lrangle{\nabla f(\x_{t-\hat{\tau}}), \nabla f_{i_t}(\x_{t-\hat{\tau}}) } } \\ \nonumber 
\le & \lrincir{\sigma_v^2 + \sigma_{\Vcal}^2+G^2} \lrincir{ \frac{L^2\hat{\tau}\sum_{t=\hat{\tau}+1}^T \sum_{j=t-\hat{\tau}}^{t-1}\eta_j^2}{2\mu_{\Phi}^2} + \frac{3(L+1+L^2 \hat{\tau})\sum_{t=\hat{\tau}+1}^T\eta_t^2}{2\mu_{\Phi}^2} + 3G^2\sum_{t=\hat{\tau}+1}^T\eta_t^2} \\ \label{equa:J1}
& + f(\x_{\hat{\tau}+1}) - f(\x_{T+1}) + \frac{\hat{\tau}G^2\sum_{t=\hat{\tau}+1}^T\eta_t^2}{2}.
\end{align} 

Now, we start to bound $J_2$. 
\begin{align}
\nonumber
& J_2 \le \EE \sum_{t=\hat{\tau}+1}^T \eta_t \lrnorm{\nabla f(\x_{t-\hat{\tau}})} \lrnorm{\sum_{i=1}^n \nabla f_i(\x_{t-\hat{\tau}}) \cdot \lrincir{\frac{1}{n} - [\P^{\hat{\tau}}]_{i_{t-\hat{\tau}}, i}}} \\ \nonumber
\le & \EE \sum_{t=\hat{\tau}+1}^T \eta_t \lrnorm{\nabla f(\x_{t-\hat{\tau}})} \lrnorm{\nabla f_{j_{\max}}(\x_{t-\hat{\tau}})} \abs{\sum_{i=1}^n  \lrincir{\frac{1}{n} - [\P^{\hat{\tau}}]_{i_{t-\hat{\tau}}, i}}} {~~~~} (\text{where } j_{\max} \in \argmax_{j\in[n]} \lrnorm{\nabla f_j(\x_{t-\hat{\tau}})}) \\ \nonumber
\le & \EE \sum_{t=\hat{\tau}+1}^T \eta_t \lrnorm{\nabla f(\x_{t-\hat{\tau}})} \lrnorm{\nabla f_{j_{\max}}(\x_{t-\hat{\tau}})} \lrnorm{\bPi_\ast - \P^{\hat{\tau}}}_{\infty} \\ \nonumber
\le & \EE \sum_{t=\hat{\tau}+1}^T \frac{\eta_t}{2} \lrincir{\lrnorm{\nabla f(\x_{t-\hat{\tau}})}^2 + \lrnorm{\nabla f_{j_{\max}}(\x_{t-\hat{\tau}})}^2} \lrnorm{\bPi_\ast - \P^{\hat{\tau}}}_{\infty} \\ \nonumber
\le & \EE \sum_{t=\hat{\tau}+1}^T \frac{\eta_t}{2} \lrincir{\lrnorm{\nabla f(\x_{t-\hat{\tau}})}^2 + \lrnorm{\nabla f_{j_{\max}}(\x_{t-\hat{\tau}}) - \nabla f(\x_{t-\hat{\tau}}) + \nabla f(\x_{t-\hat{\tau}})}^2} \lrnorm{\bPi_\ast - \P^{\hat{\tau}}}_{\infty} \\ \nonumber
\le & \EE \sum_{t=\hat{\tau}+1}^T \frac{\eta_t}{2} \lrincir{\lrnorm{\nabla f(\x_{t-\hat{\tau}})}^2 + 2\lrnorm{\nabla f_{j_{\max}}(\x_{t-\hat{\tau}}) - \nabla f(\x_{t-\hat{\tau}})}^2 + 2\lrnorm{\nabla f(\x_{t-\hat{\tau}})}^2} \lrnorm{\bPi_\ast - \P^{\hat{\tau}}}_{\infty} \\ \nonumber
\le & \lrincir{\frac{3G^2}{2}+\sigma_{\Vcal}^2} \sum_{t=\hat{\tau}+1}^T \eta_t \lrnorm{\bPi_\ast - \P^{\hat{\tau}}}_{\infty} \\ \label{equa:J2}
\le & \lrincir{\frac{3G^2}{2}+\sigma_{\Vcal}^2} \sum_{t=\hat{\tau}+1}^T \eta_t \cdot \alpha(t).
\end{align}  
Putting \ref{equa:J1} and \ref{equa:J2} back into \ref{equa:gradf:xttaut:norm}, we obtain
\begin{align}
\nonumber
& \EE \sum_{t=\hat{\tau}+1}^T \eta_t \lrnorm{\nabla f(\x_{t-\hat{\tau}})}^2 \\ \nonumber 
\le & \lrincir{\sigma_v^2 + \sigma_{\Vcal}^2+G^2} \lrincir{ \frac{L^2\hat{\tau}\sum_{t=\hat{\tau}+1}^T \sum_{j=t-\hat{\tau}}^{t-1}\eta_j^2}{2\mu_{\Phi}^2} + \frac{3(L+1+L^2 \hat{\tau})\sum_{t=\hat{\tau}+1}^T\eta_t^2}{2\mu_{\Phi}^2} + 3G^2\sum_{t=\hat{\tau}+1}^T\eta_t^2} \\ \label{equa:J1}
& + f(\x_{\hat{\tau}+1}) - f(\x_{T+1}) + \frac{\hat{\tau}G^2\sum_{t=\hat{\tau}+1}^T\eta_t^2}{2} + \lrincir{\frac{3G^2}{2}+\sigma_{\Vcal}^2} \sum_{t=\hat{\tau}+1}^T \eta_t \cdot \alpha(t).
\end{align} It thus completes the proof.
\end{proof}

\bibliographystyle{unsrtnat}
\bibliography{references}  







\end{document}